\documentclass[conference]{IEEEtran}

\usepackage{etoolbox}
\makeatletter
\patchcmd{\@makecaption}
  {\scshape}
  {}
  {}
  {}
\makeatletter
\patchcmd{\@makecaption}
  {\\}
  {.\ }
  {}
  {}
\makeatother

\usepackage{verbatim}
\usepackage{cite}
\usepackage{amsmath,amssymb,amsfonts}
\usepackage{graphicx}
\usepackage{multirow}
\usepackage{textcomp}
\usepackage{xcolor}
\usepackage{mathtools}
\usepackage{float}
\usepackage[utf8]{inputenc}
\usepackage{xparse}
\ifCLASSOPTIONcompsoc \usepackage[caption=false,font=normalsize,labelfont=sf,textfont=sf]{subfig}
\else
\usepackage[caption=false,font=footnotesize]{subfig}
\fi

\usepackage[ruled,lined, noend,linesnumbered]{algorithm2e}
\usepackage{algorithmicx}
\usepackage{titlesec}
\usepackage[acronym,toc]{glossaries}
\usepackage{hyperref}

\usepackage{amsthm}
\newtheorem{lemma}{Lemma}

\usepackage[table]{xcolor}
\usepackage{colortbl}

\definecolor{best}{RGB}{204,255,204}    % light green
\definecolor{worst}{RGB}{255,204,204}   % light red
\definecolor{mid}{RGB}{255,255,204}     % light yellow

\makeglossaries

\newglossaryentry{tula}{
  name={Tula},
  description={Optimizing Performance, Cost and Generalization in Large-Batch Training}
}

\newacronym{dnn}{DNN}{Deep Neural Networks}
\newacronym{hpc}{HPC}{High Performance Computing}
\newacronym{gpu}{GPU}{Graphic Processing Unit}
\newacronym{singpu}{GPU}{Graphic Processing Unit}
\newacronym{sgd}{SGD}{Stochastic Gradient Descent}
\newacronym{kde}{KDE}{Kernel Density Estimates}
\newacronym{vm}{VM}{virtual machines}
\newacronym{ags}{AGS}{Adaptive Gradient Scaling}
\newacronym{ddp}{DDP}{Distributed Data-Parallel}
\newacronym{lr}{LR}{learning-rate}
\newacronym{oom}{OOM}{out-of-memory}
\newacronym{lars}{LARS}{Layerwise adaptive rate scaling}

\def\BibTeX{{\rm B\kern-.05em{\sc i\kern-.025em b}\kern-.08em
    T\kern-.1667em\lower.7ex\hbox{E}\kern-.125emX}}
    
\graphicspath{{./}}
    
\begin{document}

	\title{Tula: Optimizing Time, Cost, and Generalization in Distributed Large-Batch Training}
	
	\author{
    \IEEEauthorblockN{Sahil Tyagi}
    \IEEEauthorblockA{
        Oak Ridge National Laboratory\\
        Oak Ridge, Tennessee, USA \\
        tyagis@ornl.gov}
    \and
    \IEEEauthorblockN{Feiyi Wang}
    \IEEEauthorblockA{
        Oak Ridge National Laboratory\\
        Oak Ridge, Tennessee, USA \\
        fwang2@ornl.gov}
}
	
	\maketitle
	
	\begin{abstract}
		\begin{comment}
Distributed training processes more batches per-iteration either by scaling-out (adding more nodes) or scaling-up (by increasing batch-size).
However, using the largest configuration may not always yield best results; horizontal scaling incurs additional communication cost while vertical scaling is limited by the computational overhead and device memory ceilings.
Thus, simply scaling the batch-size offers diminishing returns, where training time/cost decreases initially but eventually saturates (i.e., the knee-point of time/cost vs. batch-size pareto curve).
The optimal batch-size thus depends on the \acrshort{dnn}, training dataset and the available compute resources.
Additionally, the statistical performance of large-batches is worse than small-batches due to ``generalization gap''.
In this paper, we introduce \gls{tula}, an online, black-box service to optimize time, cost and convergence quality of large-batch training for vision-based convolutional models. It does so by developing parallel and statistical performance models to identify the optimal batch-size based on user objective.
In our evaluation, \gls{tula} predicts the trend in training time/cost within 7.5-14\% error-rate across multiple \acrshort{dnn}s.
Over a na\"ively chosen configuration, \gls{tula} achieves up to 20$\times$ speedup by identifying the optimal batch-size, while achieving $\approx$ 8.8\% more accuracy than vanilla large-batch training on average, thus improving generalizability.
\end{comment}
Distributed training increases the number of batches processed per iteration either by scaling-out (adding more nodes) or scaling-up (increasing the batch-size). However, the largest configuration does not necessarily yield the best performance.
Horizontal scaling introduces additional communication overhead, while vertical scaling is constrained by computation cost and device memory limits.
Thus, simply increasing the batch-size leads to diminishing returns: training time and cost decrease initially but eventually plateaus, creating a knee-point in the time/cost vs. batch-size pareto curve.
The optimal batch-size therefore depends on the underlying model, data and available compute resources.
Large batches also suffer from worse model quality due to the well-known ``generalization gap''.
In this paper, we present \gls{tula}, an online service that automatically optimizes time, cost, and convergence quality for large-batch training of convolutional models. It combines parallel-systems modeling with statistical performance prediction to identify the optimal batch-size.
%that best satisfies a user-specified objective. 
\gls{tula} predicts training time and cost within 7.5–14\% error across multiple models, and achieves up to 20$\times$ overall speedup and improves test accuracy by $\approx$ 9\% on average over standard large-batch training on various vision tasks, thus successfully mitigating the generalization gap and accelerating training at the same time.
	\end{abstract}
	
	\section{Introduction}\label{sec:intro}

Distributed processing is essential to develop accurate models trained over vast datasets~\cite{b0}.
In the age of big data and federated processing, training deep learning workloads spans a broad spectrum of hardware—from datacenter-scale \glspl{gpu} to edge and personal devices.
This ecosystem is highly heterogeneous, with compute~\cite{b5} and network capabilities~\cite{b3} varying significantly across edge~\cite{b72}, cloud, and \gls{hpc} environments.
Consequently, existing training techniques and systems are often unable to utilize/allocate resources efficiently for a given task.
Choosing an appropriate training and system configuration has a pronounced impact on both the parallel and statistical performance of neural networks~\cite{b6,b28}.

The pareto relationship between parallel and statistical efficiency is influenced by total nodes, hardware, batch-size, data and execution characteristics of the model~\cite{b8,b9,b28}.
A sub-optimal configuration may lead to extended runtimes, increased cost, or under-utilization of resources.
Unlike traditional distributed processing jobs, model training involves unique execution and synchronization behaviors, along with numerous tunable parameters that directly affect model quality, performance and efficiency.
Scaling efficiency is affected by network topology, latency, bandwidth, cluster-size, batch-size, model-size, and communication protocol~\cite{b3}, while statistical performance is influenced by communication frequency, data distribution~\cite{b42} and training parameters such as \gls{lr} and batch-size.
Large batch-sizes specially suffer from worse test performance due to ``generalization gap''~\cite{b12,b14}.

In this work, we present \gls{tula}, a black-box, online service that develops practical and principled techniques for performance and cost-aware training while mitigating the generalization gap associated with large-batches.
To enable systematic optimization, we construct an empirical performance model for large-batch distributed training: profiling-based compute analysis to estimate memory demands, coupled with a parallel model to predict compute and synchronization cost across different cluster and batch-sizes.
Beyond classical parallel scaling considerations, \gls{tula} incorporates fundamental characteristics of \gls{sgd} based optimization to address the generalization deficit of large-batches by scaling gradient updates.
\gls{tula} is tested with convolutional models across various vision tasks and compared with multiple baselines.
%We evaluate \gls{tula} over convolutional models across various vision tasks and compare with multiple baselines.
In general, we make the following contributions:
\begin{itemize}
\item Perform a comprehensive analysis of the trade-offs between parallel and statistical efficiency in large-batch distributed training.
\item Develop a practical performance model that estimates resource requirements, execution overhead, and end-to-end runtime cost for a given model, dataset, training hyperparameters, and compute resources.
\item Theoretically analyze the proposed gradient-scaling technique that projects large-batch updates to small-batch space, significantly improving model accuracy over vanilla \acrshort{sgd} and achieving similar or \emph{better} performance as other large-batch optimizers.
\item Evaluate multiple models to demonstrate \gls{tula}'s effectiveness in quantifying time and cost improvements over na\"ive configurations under diverse user objectives (e.g., minimizing time, cost, or a more balanced combination), and exploring multiple cost models for selecting the optimal batch-size in distributed training.
\end{itemize}

	\section{Background and Challenges}\label{sec:background}

\textit{Performance trade-offs in distributed training:} 
In this section, we characterize the performance trade-offs inherent in distributed training.
These observations inform the design of the parallel and statistical performance models in \S\ref{sec:design} to improve large-batch training efficiency.

%\vspace{0.1cm}
\subsection{Parallel Efficiency in Distributed Training}\label{subsec:scaling}
Neural networks are trained by iteratively optimizing model parameters by minimizing a loss function using stochastic methods like \acrshort{sgd}.
This compute-intensive process is commonly accelerated via data parallelism where multiple nodes maintain a full model replica and compute gradients over a disjoint subset of the data, followed by synchronization and aggregation of updates.
For an $L$-smooth objective function $f : \mathbb{R}^{d} \rightarrow \mathbb{R}$ satisfying Equation~(\ref{eqn:Lsmooth}), \acrshort{sgd} on each worker produces an unbiased estimator of the true gradient over random samples $b$ with bounded variance $\sigma^{2}$ (Equation~(\ref{eqn:boundedVar})).
In synchronous \acrshort{sgd}, updates are averaged over $N$ nodes and applied as per Equation~(\ref{eqn:sgdupdate}) with convergence rate $\mathcal{O}(1/\sqrt{I})$ for non-convex stochastic optimization over $I$ iterations~\cite{b32}.
Updates are aggregated through a central parameter server or decentralized collective communication like all-reduce, each exhibiting distinct latency-bandwidth costs~\cite{b3}.

\begin{subequations}
	\begin{equation}
		f(x) \leq f(y) + \langle \nabla f(y), x-y \rangle + \dfrac{L}{2}||x-y||^{2} \;\; \forall \;\; x,y
		\label{eqn:Lsmooth}
	\end{equation}
	\begin{equation}
		\mathbb{E}[g^{(i)} | w^{(i)}] = \nabla f(w^{(i)}) \Rightarrow \mathbb{E}||g^{(i)} - \nabla f(w^{(i)})||^{2} \leq \sigma^{2} \; \forall \; i
		\label{eqn:boundedVar}
	\end{equation}
	\begin{equation}
		w^{(i + 1)}_{(n)} = w^{(i)}_{(n)} - \eta \cdot \dfrac{1}{N} \sum_{n=1}^{N} (\frac{1}{|b|}\sum_{b_{(i)} \in \mathcal{B}^{(n)}_{(i)}} \nabla f(w^{(i)}_{(n)}, b_{(i)}))
		\label{eqn:sgdupdate}
	\end{equation}
	\label{eqn:sgdprereqs}
\end{subequations}

Distributed training can be scaled either \textit{vertically} or \textit{horizontally}.
We raise each worker's batch-size in vertical scaling, increasing the per-step compute time along with memory consumption due to additional forward passes and larger activation maps.
Horizontal scaling adds more nodes to scale-out training that may result in extra synchronization overhead~\cite{b19}.
With a fixed local batch-size $b$, the global batch ($B$ = $Nb$) increases proportional to cluster-size in the latter, while vertical scaling over a fixed $N$ can scale to larger $B$ by increasing the local batch-size.
These competing effects result in complex trade-offs between computation and communication efficiency.
%with the two scaling mechanisms (varying cluster or local batch-size).

\begin{figure*}
\centering
\subfloat[ResNet50]{\includegraphics[width=0.25\textwidth]{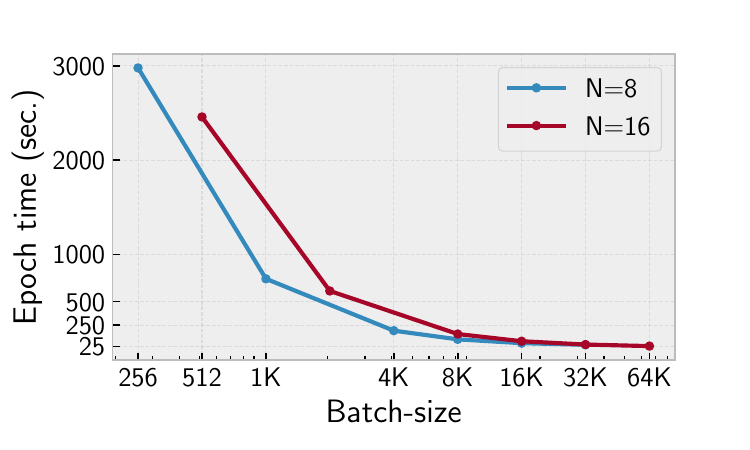}}
\subfloat[VGG11]{\includegraphics[width=0.25\textwidth]{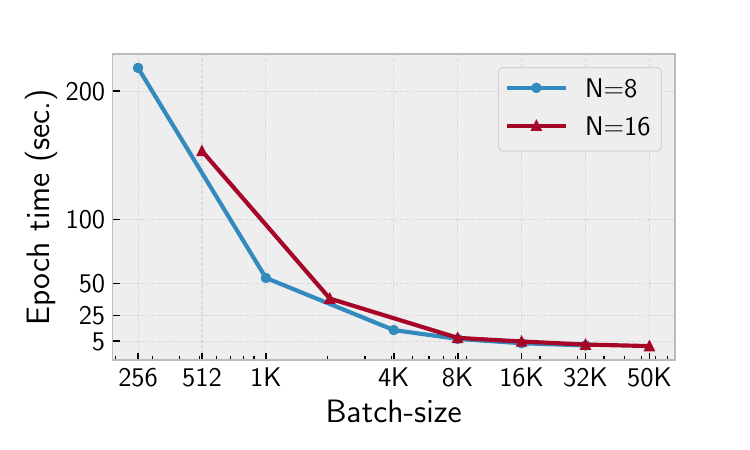}}
\subfloat[AlexNet]{\includegraphics[width=0.25\textwidth]{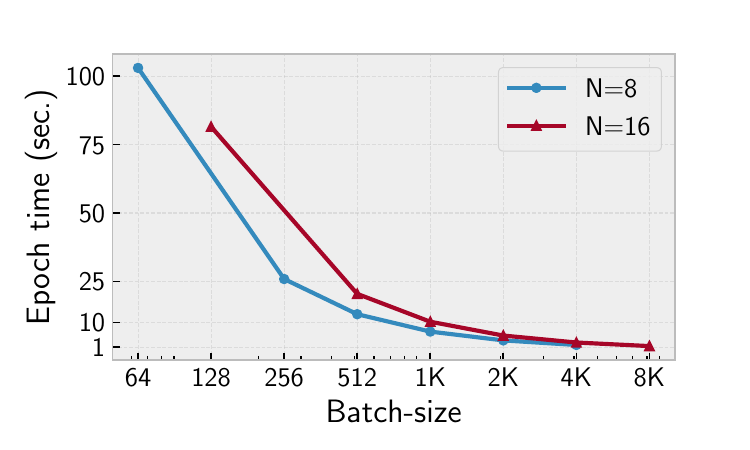}}
\subfloat[MobileNetv3]{\includegraphics[width=0.25\textwidth]{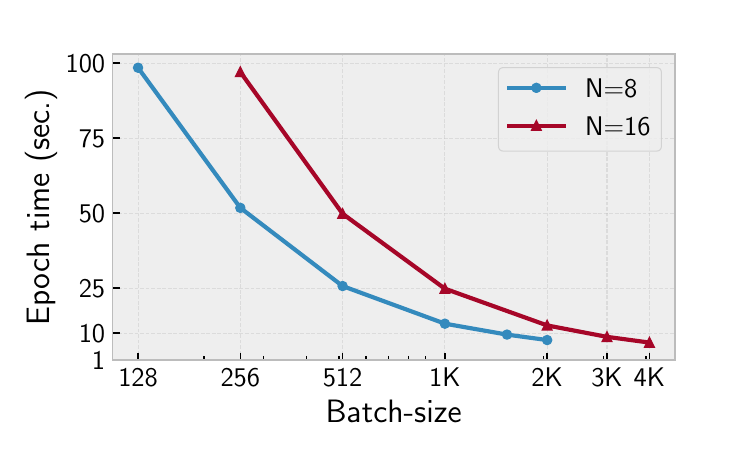}}
\caption{Epoch time of ResNet50 on ImageNet, VGG11 on CIFAR100, AlexNet on CalTech101, and MobileNetv3 on CalTech256 across two cluster-sizes.}
\label{fig:fig1EpochTime}
\end{figure*}

Figure~(\ref{fig:fig1EpochTime}) reports epoch completion time for different convolutional networks as the global batch-size is scaled over 8 and 16 V100 \gls{gpu}s.
An epoch requires $|D/B|$ steps for dataset $D$ and global batch $B$, so increasing the batch-size reduces the total steps needed to complete each epoch.
However, this increases the per-iteration cost (compute in vertical and communication in horizontal scaling), resulting in diminishing performance gains.
Across all models, epoch time decreases sharply initially but gradually plateaus beyond a model-dependent threshold (e.g., 8192 for ResNet50~\cite{b24}, 4096 in VGG11~\cite{b25}, 2048 for AlexNet~\cite{b26} and MobileNetv3~\cite{b27}).
For a fixed cluster size $N$, the epoch time decreases as batch size increases, since fewer iterations are required to process the entire dataset, despite the higher per-iteration computational cost of larger local batches. However, as the global batch size continues to grow (toward the right end of the x-axis), intra-device parallelism reaches its limit, causing the epoch completion time to plateau.
For the same batch-size, fewer workers at $N$=8 incur lower communication cost and thus achieve faster epoch times.
For e.g., MobileNetv3 at 1K completed an epoch in \emph{more than twice the time} with 16 nodes compared to a cluster of 8, so reduced synchronization cost outweighed the increased per-device compute cost here.
%For e.g., MobileNetv3 at 1K completed an epoch in 12 and 25 seconds with 8 and 16 nodes respectively, so reduced synchronization cost outweighed the increased per-device compute cost here.
However, as batch-size rises, memory constraints and reduced intra-worker parallel efficiency limit scalability, causing epoch times to converge across clusters.
Under weak scaling, local batch-size is fixed and global batch rises with cluster-size, resulting in fewer steps per-epoch and improving performance despite extra communication.
For e.g., ResNet50 with local batch-size 32 completes an epoch $\approx$ $3\times10^{3}$ seconds with 8 nodes and about $2.4\times10^{3}$ seconds with 16, about 25\% faster over the larger cluster (both thus using the same global batch-size of 256).

\emph{\textbf{Takeaway:} Distributed training exhibits non-linear scalability and diminishing returns from the interplay between computation, communication, and memory constraints.}

\subsection{Statistical Efficiency in Distributed Training}\label{subsec:bggengap}
\begin{figure*}
\centering
\subfloat[ResNet50]{\includegraphics[width=0.25\textwidth]{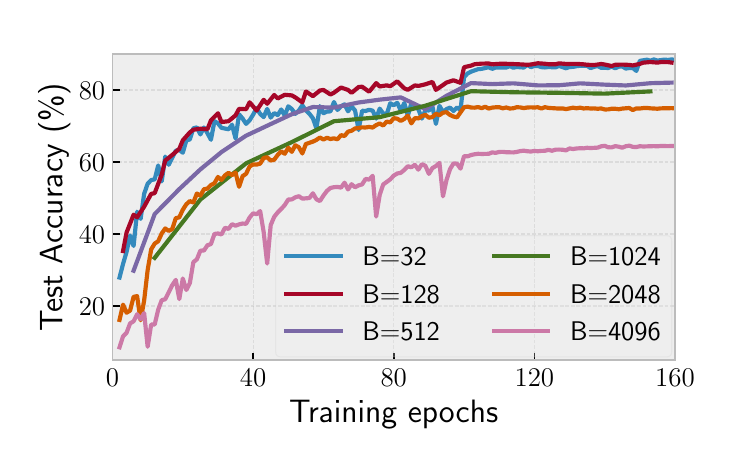}}
\subfloat[VGG11]{\includegraphics[width=0.25\textwidth]{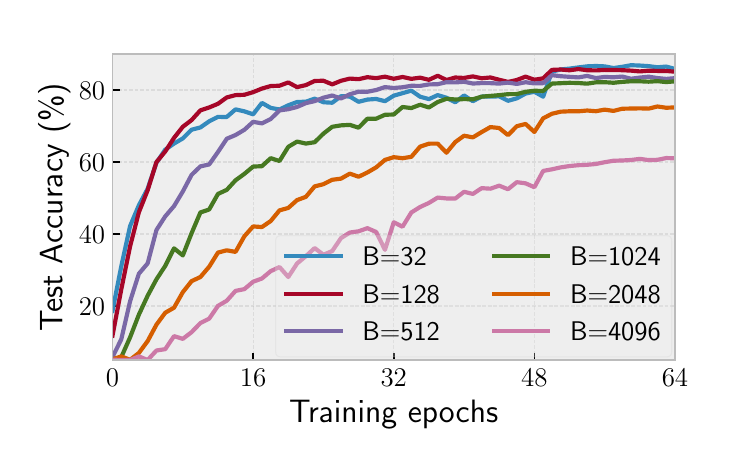}}
\subfloat[AlexNet]{\includegraphics[width=0.25\textwidth]{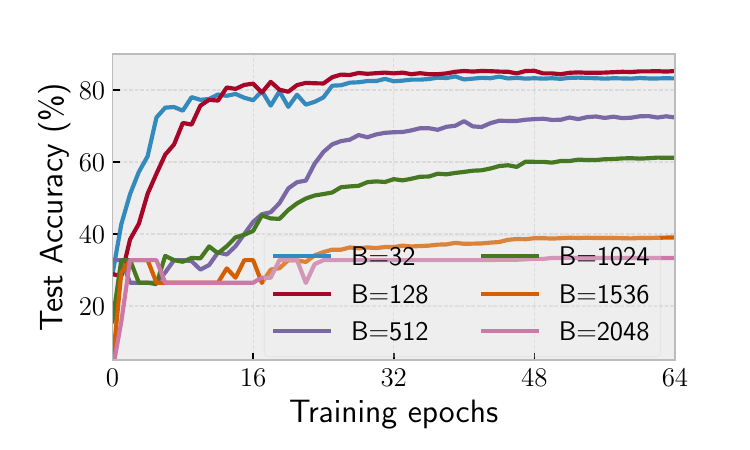}}
\subfloat[MobileNetv3]{\includegraphics[width=0.25\textwidth]{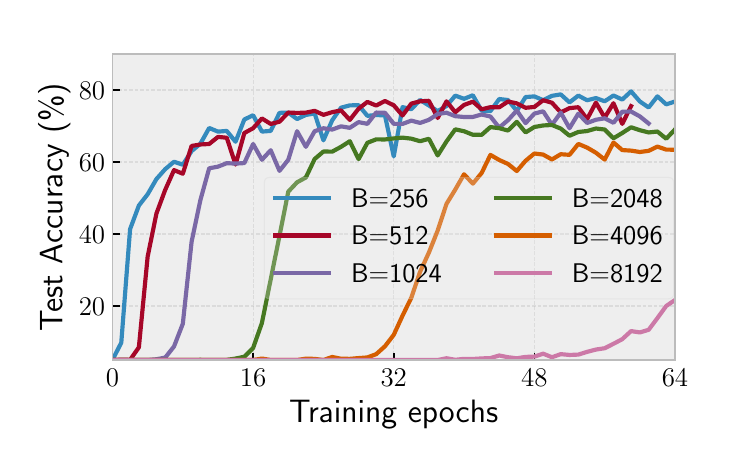}}
\caption{Test accuracy vs. batch-size. Smaller batches achieve better test accuracy than large-batch training, illustrating generalization gap in the latter.}
\label{fig2GenGapLargeBatch}
\end{figure*}

Although large-batch training can significantly improve throughput and achieve high training accuracy, it often suffers from degraded statistical performance on held-out data, commonly referred to as generalization gap~\cite{b14,b34,b35}.
Prior studies attribute this phenomenon to tendency of large-batches to converge to sharper minima in the optimization landscape~\cite{b12}, whereas smaller batches introduce stochasticity that biases optimization toward flatter minima with a narrower Hessian spectrum~\cite{b36}, empirically associated with improved generalization.
We see this from the test accuracy of multiple models in Figure~(\ref{fig2GenGapLargeBatch}) by varying the global batch $B$ for a fixed cluster-size.
Across all models, generalization performance consistently declines as the training batch-size increases.
The best accuracy of ResNet50 was attained at batch-size 32 and 128, albeit with lower training throughput due to limited parallelism.
Increasing $B$ to 4096 results in progressively worse accuracy.
VGG11 observed a similar trend in test performance for the same batches.
AlexNet exhibited even greater sensitivity to large-batch training, where accuracy dropped from 82\% to 37\% as batch-size was scaled from 128 to 2048.
Similarly, MobileNetv3 degraded from 80\% to 20\% accuracy as we increased the batch-size from 256 to 8192.
Beyond the batches shown in the figure, all evaluated models experienced severe generalization loss, rendering further batch scaling impractical.

\emph{\textbf{Takeaway:} Large batches improve parallel efficiency, they can substantially degrade convergence quality.
Consequently, any system that optimizes distributed training performance must jointly account for both parallel and statistical efficiency.}

	\section{Design}\label{sec:design}

\begin{figure}[t]
    \centering
    \includegraphics[width=\columnwidth]{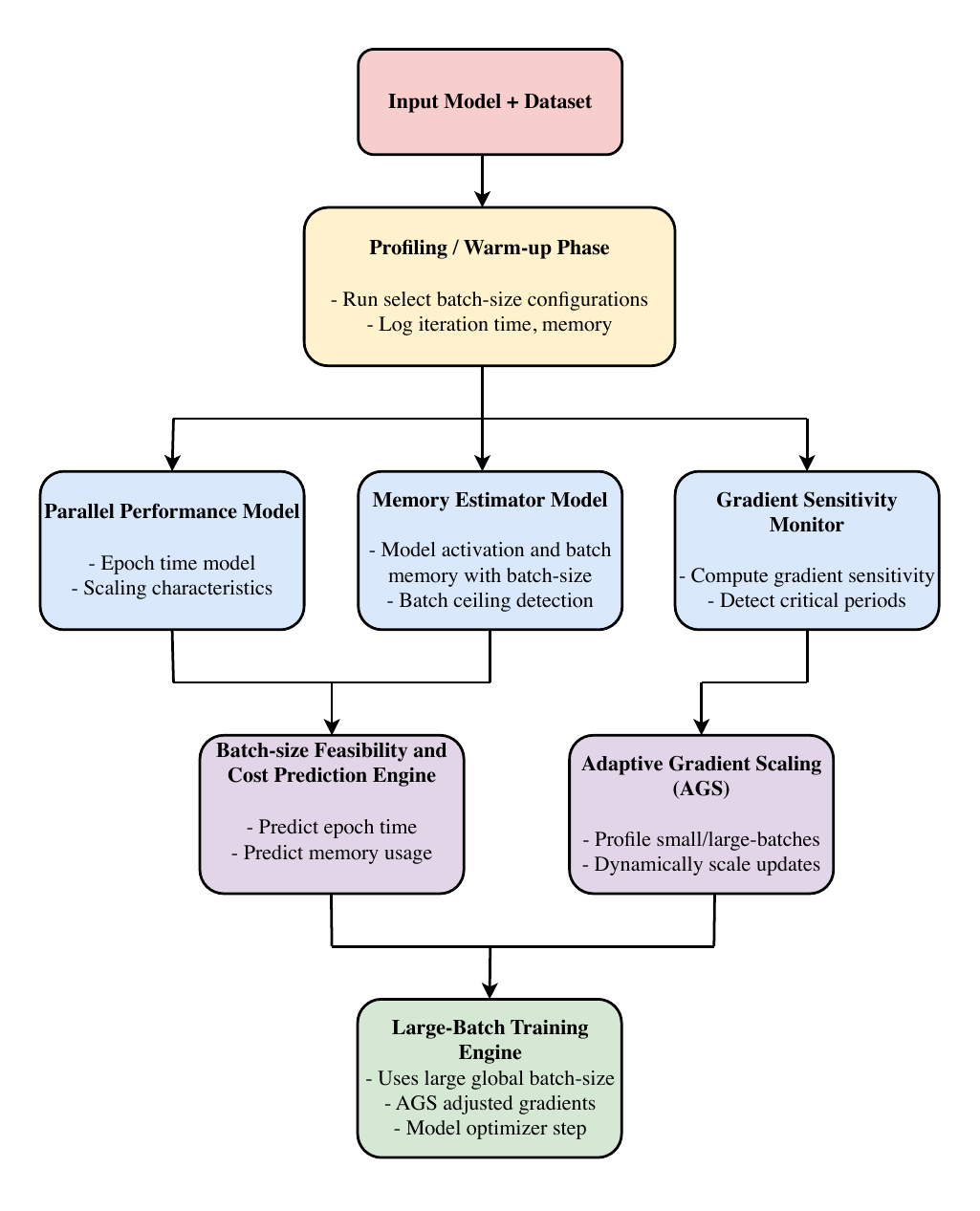}
    \caption{A schematic overview of Tula's workflow.}
    \label{fig:tuladesign}
\end{figure}

Parallel training across distributed nodes yields substantial speedups, albeit with diminishing returns as we scale further.
The optimal configuration depends on multiple system and workload characteristics, including network topology, latency and bandwidth, model-size, cluster-size, batch-size, and hardware resources.
Even with a systems-optimal large-batch configuration, distributed training still suffers from statistical inefficiency due to generalization gap.
\gls{tula} jointly optimizes parallel and statistical efficiency by first identifying the optimal batch and cluster-size configuration that minimizes training time, cost or the knee-point of the trade-off curves, while simultaneously improving large-batch generalization.
The overall logical flow of Tula is illustrated in Figure~(\ref{fig:tuladesign}), with each block described in the following sections.

\subsection{Parallel Performance Model}\label{subsec:designparperf}
From a computational standpoint, \gls{sgd}-based optimizations are iterative and repetitive, with each iteration $i$ loading, processing and moving $b$ mini-batch of data, compute loss and gradients, then aggregate and apply updates.
Consequently, the per-step time $t_{step}$ can be decomposed into computation~\cite{b15} and communication time in synchronous training:

\begin{subequations}
	\begin{equation}
		t_{step}^{(i)} = t_{compute}^{(i)} + t_{sync}^{(i)}
		\label{eqn:itrtime}
	\end{equation}
	\begin{equation}
		t_{comp} \propto b
		\label{eqn:computecost}
	\end{equation}
	\begin{equation}
		t_{sync} \propto 
		\begin{cases}
			N, \;\; \text{centralized parameter server} \\
			(1 - \dfrac{1}{N}), \;\; \text{ring-allreduce}
		\end{cases}
		\label{eqn:synccost}
	\end{equation}
	\label{eqn:paralperfmodel}
\end{subequations}

From \S\ref{subsec:scaling}, scaling vertically incurs extra $t_{compute}$ due to larger local batches, while horizontal scaling affects $t_{sync}$ due to additional nodes.
For data-parallel training in modern cloud and data-centers, communication cost is dominated by bandwidth rather than latency~\cite{b3,b19}.
Centralized parameter servers incurs bandwidth cost that scales linearly with the number of nodes, while decentralized collectives like ring-allreduce are bandwidth-optimal.
Although other collectives (e.g., tree-allreduce, reduce-scatter-allgather, etc.) have different latency–bandwidth trade-offs, \gls{tula} accommodates by incorporating its respective synchronization models into $t_{sync}$.

\begin{equation}
	M_{total} = M_{param} + M_{grad} + M_{opt} + M_{act} + M_{batch}
	\label{eqn:DNNMemory}
\end{equation}

\subsubsection{Memory Estimation Model:}
While a large batch-size takes fewer iterations to complete training epochs and improve throughput, the largest feasible batch-size is constrained by the available memory.
The overall memory footprint of convolutional networks is approximated in Equation~(\ref{eqn:DNNMemory}).
Here, the parameter, gradient and optimizer-state memory, $M_{param}$, $M_{grad}$ and $M_{opt}$, are largely static for a given model and precision format~\cite{b45,b46}.
In contrast, activation memory $M_{act}$ and batch memory $M_{batch}$ rise with training batch-size and often become the limiting factor.
Although intermediate checkpointing lowers the activation memory cost, it incurs additional compute cost~\cite{b46}.

\begin{figure}
	\centering
	\subfloat[Activation memory]{\includegraphics[width=0.25\textwidth]{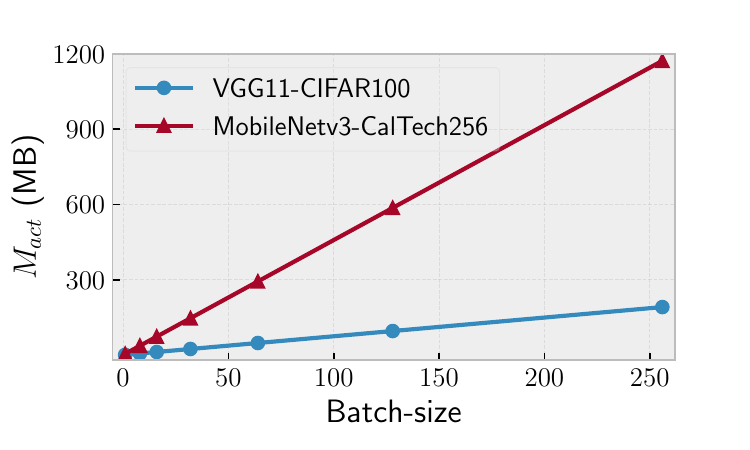}
	\label{fig6actmem}}
	\subfloat[Batch memory]{\includegraphics[width=0.25\textwidth]{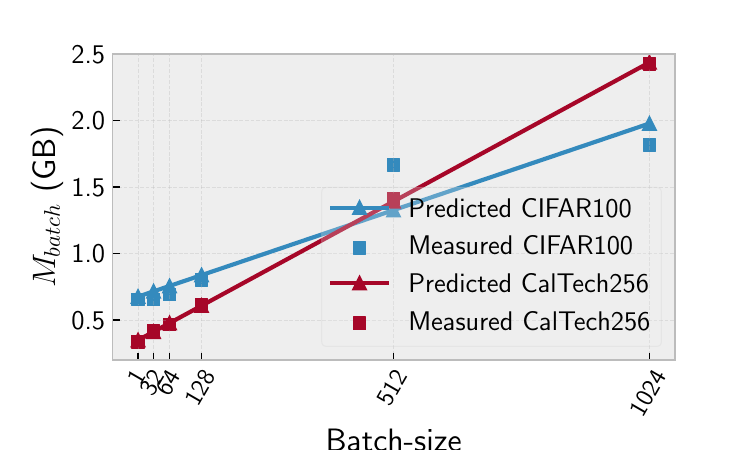}
	\label{fig6batchmem}}
	\caption{(a) $M_{act} \propto$ batch-size. (b) Linear model to predict $M_{batch}$.} 
	\label{fig6mempred}
\end{figure}

From Figure~(\ref{fig6mempred}), $M_{act}$ scales linearly with the batch-size for representative models, and $M_{batch}$ can be accurately approximated using a simple linear model for different datasets.
The scatter points in Figure~(\ref{fig6batchmem}) represents the actual dataloader memory for CIFAR100 and CalTech256 datasets at various batches, while the line plot signifies the estimated memory predicted using a linear model with respect to batch-size.
%Please note the sharp rise in batch-memory from 128 to 512 is still near-linear, batch-size quadruples here 
By leveraging Equation~(\ref{eqn:DNNMemory}) and lightweight profiling, \gls{tula} can estimate the maximum batch-size supported by a given model, dataset, and hardware, thus bounding the feasible search space for optimization.

\subsubsection{Performance and Cost Modeling:}
The objective of parallel performance model is to infer time and cost of training as a function of cluster and batch-size to estimate trade-off curves like Figure~(\ref{fig:fig1EpochTime}).
%Using the memory model, \gls{tula} bounds the search space of feasible batches between $(b_{min}, b_{max})$.
%It then profiles a subset of candidate configurations briefly to fit computation and communication models based on Equation~(\ref{eqn:paralperfmodel}).
With a user-specified minimum batch-size and a memory model inferred upper bound, $(b_{min}, b_{max})$, \gls{tula} profiles a subset of candidate configurations briefly to fit computation and communication models of Equation~(\ref{eqn:paralperfmodel}).

\begin{subequations}
	\begin{equation}
		T = I \cdot \widetilde{t}_{step} \;\; \Rightarrow \;\; T = \dfrac{ED}{B} \cdot \widetilde{t}_{step}
		\label{eqn:traintimeEst}
	\end{equation}
	\begin{equation}
		C = T \cdot N \cdot \tilde{p}
		\label{eqn:trainCostEst}
	\end{equation}
	\label{eqn:trainCostTime}
\end{subequations}

Total training time is estimated by Equation~(\ref{eqn:traintimeEst}), where $E$ is the total epochs for $I$ steps with dataset-size $D$ and average iteration time $\widetilde{t}_{step}$.
The total cost with a static pricing model (common in the cloud) is given by Equation~(\ref{eqn:trainCostEst}), where $\tilde{p}$ denotes the per-node price with $N$ nodes in total.
\gls{tula} offers two search modes: \emph{full} or \emph{partial}, where the former profiles all possible $(N,B)$ while the latter minimizes exploration cost by profiling only the extremum configurations (min-max cluster and batch-size) and extrapolates via regression.
While partial-search substantially lowers profiling cost, it incurs estimation errors due to limited samples, a trade-off we evaluate in \S\ref{sec:eval}.

\subsubsection{Configuration Selection Policy:}
Given a job and available resources, \gls{tula} selects the batch and cluster-size configuration that best satisfies a user-specified objective like minimizing training time, cost, or selecting the trade-off curve's knee-point (identified using the Kneedle algorithm~\cite{b50}).
\emph{In \gls{tula}, the feasible search space is specified by user-defined constraints on the minimum batch and cluster-size, largest available cluster-size, combined with memory-based upper batch-size bounds}.
Extremely small batches or clusters fail to exploit available parallelism, while overly large batches exacerbate statistical inefficiency and memory pressure.
By explicitly modeling these trade-offs, \gls{tula} enables a principled, efficient configuration selection mechanism for large-batch training.

\subsection{Gradient Sensitivity in Distributed Training}\label{subsec:gradsensitive}

\begin{figure}
	%\centering
	 \hspace*{-1.5em}
	\subfloat[ResNet50 gradient $\ell_2$-norm]{\includegraphics[width=0.25\textwidth]{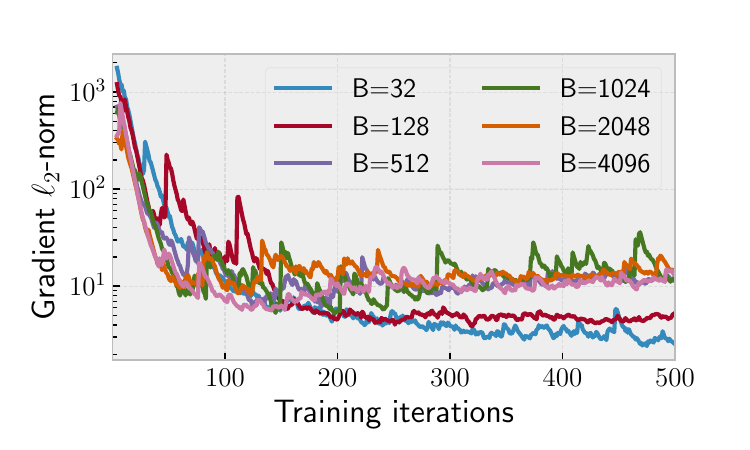}
	\label{fig3gradnormres50}}
	\subfloat[VGG11 second vs. first-order data]{\includegraphics[width=0.25\textwidth]{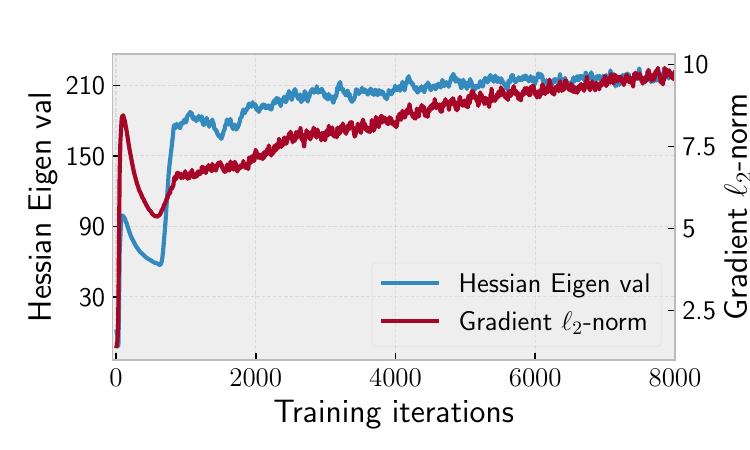}
	\label{fig4hessvgg11}}
	\caption{(a) Gradient $\ell_2$-norm in early iterations of ResNet50. (b) Largest Hessian eigenvalue and norm of the gradients over the iterations of VGG11.} 
	\label{fig3-4-hessianorm}
\end{figure}

The trajectory of model updates varies significantly throughout training due to the stochastic nature of gradient-based optimization.
Models are particularly sensitive during early training phases~\cite{b37} and during specific critical periods of convergence~\cite{b38}.
The degree of sensitivity depends on the model, dataset characteristics and optimization hyperparameters~\cite{b39}.
Consequently, not all training iterations contribute equally to convergence, motivating the need to identify and treat sensitive updates differently.
Figure~(\ref{fig3gradnormres50}) illustrates the evolution of gradient $\ell_2$-norm for various batches in ResNet50.
Initially, large-batch updates are smaller and less noisy as they more closely approximate the full-batch gradient, while smaller-batches exhibit higher variance and larger magnitudes.

As training progresses, this behavior reverses: large-batch gradients grow in magnitude and converge toward sharper minima, whereas small-batch gradients stabilize around flatter regions of the loss landscape.
Second-order information like eigenvalues of the Hessian matrix is an effective indicator of sensitive training phases~\cite{b40}, but computing it at every iteration is prohibitively expensive.
However, prior work shows that changes in Hessian eigenvalues can be approximated using first-order gradient statistics~\cite{b41}.
Figure~(\ref{fig4hessvgg11}) compares the largest Hessian eigenvalue and gradient norm in VGG11, demonstrating that abrupt changes in curvature are well reflected in first-order gradients, albeit at a significantly lower computational cost.

To quantify temporal sensitivity, we track changes in the gradient norm for batch-size $b$ over iteration $i$ as:
\begin{equation}
		\triangle(G^{(i)}_{(b)}) = \bigg| \frac{|G^{(i)}_{(b)}|^{2} - |G^{(i-1)}_{(b)}|^{2}}{|G^{(i-1)}_{(b)}|^{2}} \bigg|
		\label{eqn:relgradchange}
\end{equation}

\begin{figure}
    \hspace*{-1.5em}
    \subfloat[ResNet50]{\includegraphics[width=0.25\textwidth]{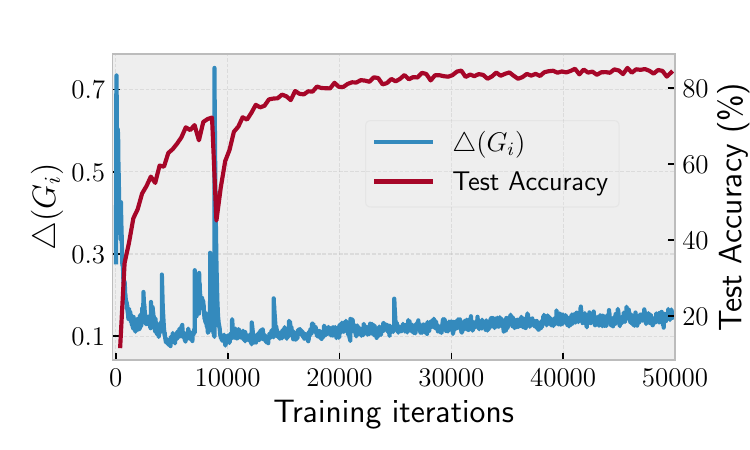}
    \label{fig5relgradres50}}
    \subfloat[VGG11]{\includegraphics[width=0.25\textwidth]{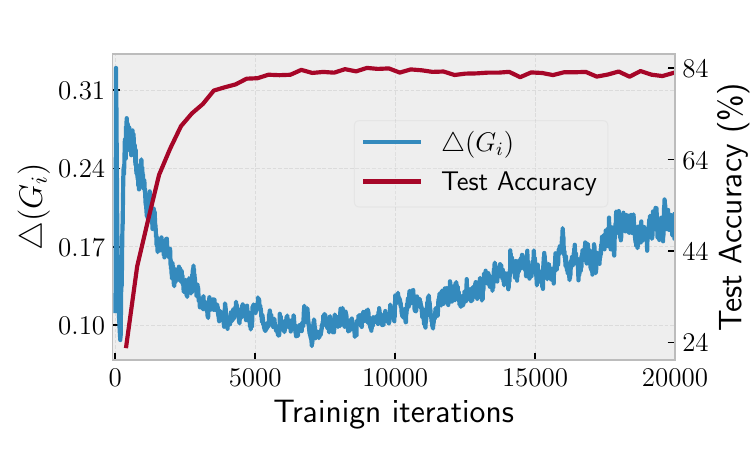}
    \label{fig5relgradvgg11}}
    \caption{Gradient variability and test accuracy are correlated such that model convergence path is mirrored in the trajectory of the rate of gradient change.}
    \label{fig5relgrad}
\end{figure}

The gradient variability metric $\triangle(G^{(i)}_{(b)})$ serves as an effective, low-cost solution to track sensitive and critical updates in an online manner during training~\cite{b41,b39}.
Figure~(\ref{fig5relgrad}) plots this metric alongside test accuracy for ResNet50 and VGG11.
In ResNet50, the rapid increase in accuracy during early training coincides with high variability in $\Delta(G^{(i)})$.
\gls{lr} decay around 9K steps induces abrupt accuracy changes that are simultaneously detected as sharp spikes in gradient change.
As training converges, both accuracy and $\Delta(G^{(i)})$ stabilize.
VGG11 exhibits a similar correlation, with steady improvements in accuracy mirrored by a gradual decline in gradient variability.
The degree of sensitivity may differ across \gls{dnn}s due to variations in architecture, depth, and dataset characteristics.
Nevertheless, $\Delta(G^{(i)})$ is able to identify critical phases within each training run by tracking relative changes in gradient magnitudes.

\emph{\textbf{Takeaway:} First-order gradient statistics give a lightweight and effective mechanism for detecting sensitive regions of the training trajectory.
Leveraging this insight, we design an adaptive gradient scaling mechanism that modulates large-batch updates based on training sensitivity, thereby mitigating the generalization gap without incurring significant computational overhead (described in \S\ref{subsec:designstatperf})}.

\begin{figure}
	\centering 
	\subfloat[ResNet50 accuracy]{\includegraphics[width=0.25\textwidth]{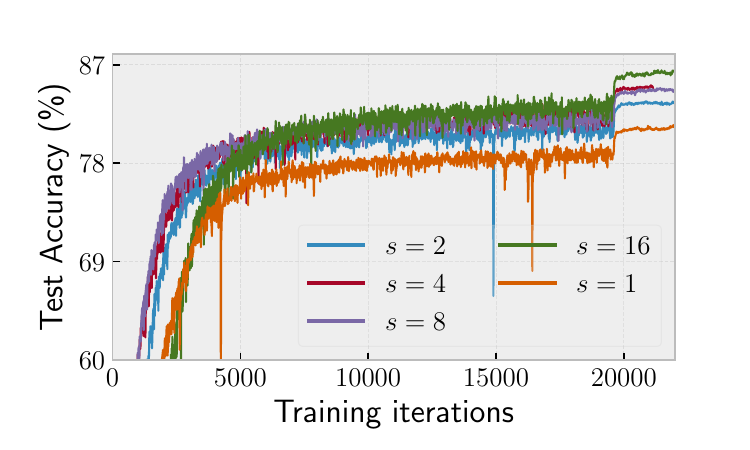}
	\label{fig:res50StepfnAcc}}
	\subfloat[ResNet50 \acrshort{kde}]{\includegraphics[width=0.25\textwidth]{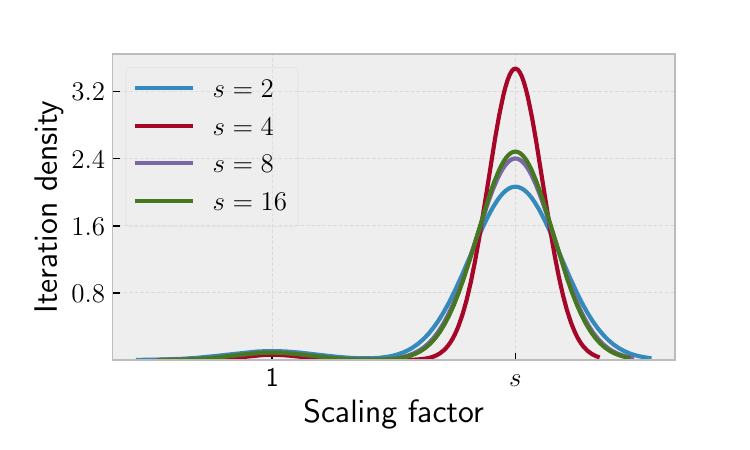}
	\label{fig:res50StepfnItrKDE}}
	\vspace{0.1cm}
	\subfloat[AlexNet accuracy]{\includegraphics[width=0.25\textwidth]{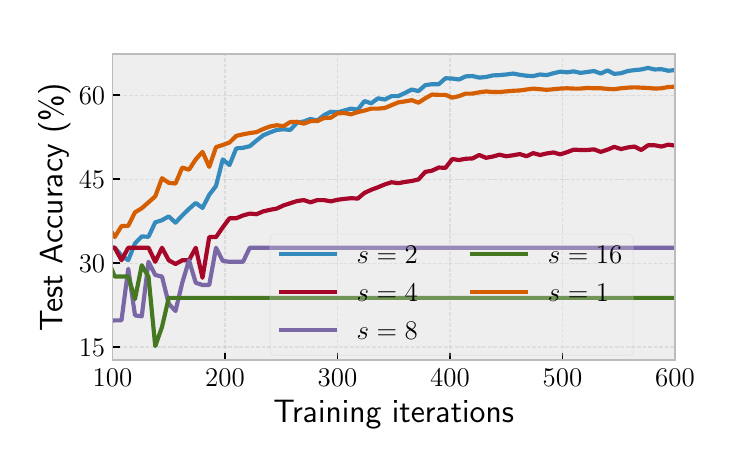}
	\label{fig:alexStepfnAcc}}
	\subfloat[AlexNet \acrshort{kde}]{\includegraphics[width=0.25\textwidth]{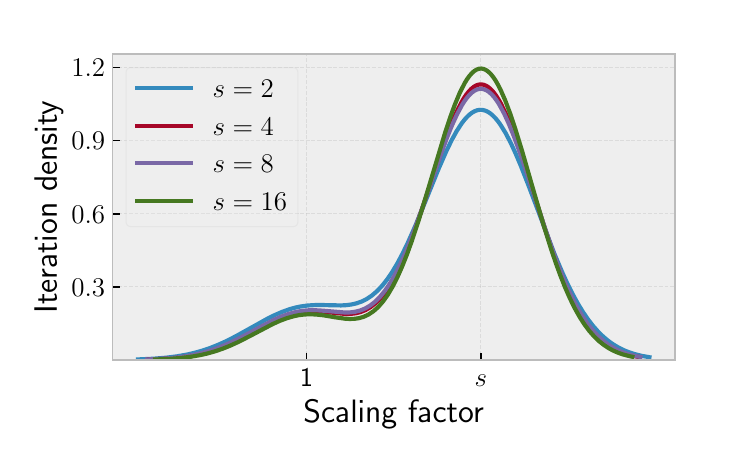}
	\label{fig:alexStepfnItrKDE}}
	\caption{Convergence quality and \acrfull{kde} of iterations for batch-size 1024 under static gradient scaling $\mathcal{U}(s)$ with $\delta=0.5$ and various scaling factors $s$.
Static scaling improves accuracy over vanilla large-batch training ($s=1$), but excessive scaling degrades convergence quality.}
	\label{fig:accuracyKDEStepFn}
\end{figure}

\subsection{Statistical Performance Model}\label{subsec:designstatperf}
At equivalent training stages, large-batch updates are generally smoother than small-batch gradients, especially during early training when optimization is most sensitive~\cite{b37}.
Gradients computed over large batches better approximate the true full-batch gradient~\cite{b8}, while small-batch updates exhibit higher variance.
While this noise can slow early convergence, it provides implicit regularization in later stages, guides models towards flatter minima and improves generalization.
In contrast, large-batch optimization often converges to sharp minima~\cite{b12}, resulting in overfitting and degraded test performance.
We model the discrepancy between small and large-batch updates at iteration $i$ as an adaptive noise term:

\begin{equation}
	G_{\tiny{b_{small}}}^{(i)} = G_{\tiny{b_{large}}}^{(i)} + \gamma^{(i)}(b_{small}, b_{large})
	\label{eqn:slgradnoiseterm}
\end{equation}

where $\gamma^{(i)}$ captures the variance induced by batch-size differences.
Rather than injecting fixed noise~\cite{b47}, we transform large-batch gradients through a mapping mechanism $\mathcal{M}(\cdot)$ that adapts to training sensitivity, as shown in Equation~(\ref{eqn:gradLSscale}).
$\mathcal{M}(\cdot)$ captures the spatio-temporal relationship between large and small-batch updates, effectively scaling large-batch gradients to approximate the regularization effect of smaller batches.
If applied judiciously, this can potentially improve convergence while retaining the efficiency of large-batch parallelization.

\begin{subequations}
	\begin{equation}
		G^{(i)}_{\tiny{b_{large}}} = \frac{1}{|b_{large}|}\nabla f(w^{(i)}, b_{large})
		\label{eqn:largeBgrad}
	\end{equation}
	\begin{equation}
		\widetilde{G}^{(i)} = \mathbf{\mathcal{M}}(G_{\tiny{b_{large}}}^{(i)}),
		\quad
		w^{(i+1)} = w^{(i)} - \eta \cdot \widetilde{G}^{(i)}
		\label{eqn:LSgradmap}
	\end{equation}
	\label{eqn:gradLSscale}
\end{subequations}

\subsubsection{Static Gradient Scaling:}
We first consider a static mapping that applies a step-wise scaling function based on gradient sensitivity, shown in Equation~(\ref{eqn:staticMapFn}).
Sensitive phases are detected using $\Delta(G^{(i)})$ from Equation~(\ref{eqn:relgradchange}) and setting a threshold $\delta$ such that updates are unscaled when $\Delta(G^{(i)})\ge\delta$ (i.e., in sensitive phases) and scaled otherwise.

\begin{subequations}
	\begin{equation}
		\mathcal{U}(s) = 
		\begin{cases}
			1, & \text{sensitive phase} \\
			s, & \text{otherwise}
		\end{cases}
		\label{eqn:stepUpFn}
	\end{equation}
	\begin{equation}
			\widetilde{G}^{(i)} = \mathcal{U}(s) \odot G_{b}^{(i)}
			\label{eqn:stepscaling}
		\end{equation}
	\label{eqn:staticMapFn}
\end{subequations}

Figure~(\ref{fig:accuracyKDEStepFn}) illustrates the statistical impact of static scaling when model updates are blindly scaled up by various $s\in[2,4,8,16]$ with $\delta$ = 0.5 and a base batch-size of 1024.
For ResNet50 in Figure~(\ref{fig:res50StepfnAcc}), convergence quality improved across all scaling factors (for any $s$ \scalebox{0.85}{\textgreater} 1) over vanilla large-batch training ($s$ = 1).
On the other hand, AlexNet saw improved accuracy over moderate scaling ($s$ = 2) but degraded performance from excess scaling ($s$ \scalebox{0.85}{\textgreater} 2).
Correspondingly, we plot the kernel density estimates (KDEs) of ResNet50 and AlexNet in Figures~(\ref{fig:res50StepfnItrKDE}) and (\ref{fig:alexStepfnItrKDE}), which quantifies the fraction of iterations that were subjected to scaling $s$ \scalebox{0.85}{\textgreater} 1 versus the portion of training iterations that used the original updates as is (i.e., $s$ = 1) based on the sensitivity of model updates.
The KDEs show how sensitivity-aware scaling preserves critical updates by switching the scaling factor between 1.0 and $s$ on a per-iteration basis.
Compared to ResNet50, AlexNet was more vulnerable to gradient sensitivity for every $s$, as seen from its KDE estimates where a considerable portion of updates default to the base scaling of 1.0.
This high degree of sensitivity translates to its convergence quality as well, where higher $s$ failed to outperform vanilla large-batch training ($s$ = 1).
In contrast, ResNet50 was more robust to gradient sensitivity across all $s$, and most iterations trained with the scaled factor $s$.
Due to its low sensitivity, ResNet50 converged successfully for every $s$ \scalebox{0.85}{\textgreater} 1 that we evaluated.

\begin{figure}
	\centering
	\subfloat[ResNet50 with $\delta$ = 0.8]{\includegraphics[width=0.25\textwidth]{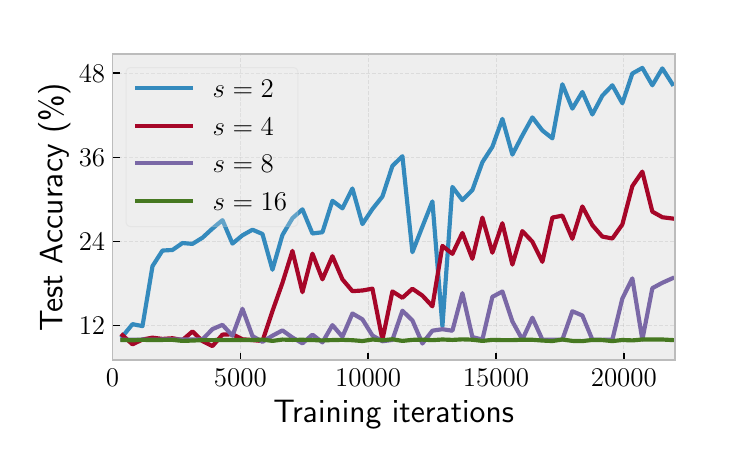}
	\label{fig8ablares50acc}}
	\subfloat[VGG11 with $\delta$ = 0.2]{\includegraphics[width=0.25\textwidth]{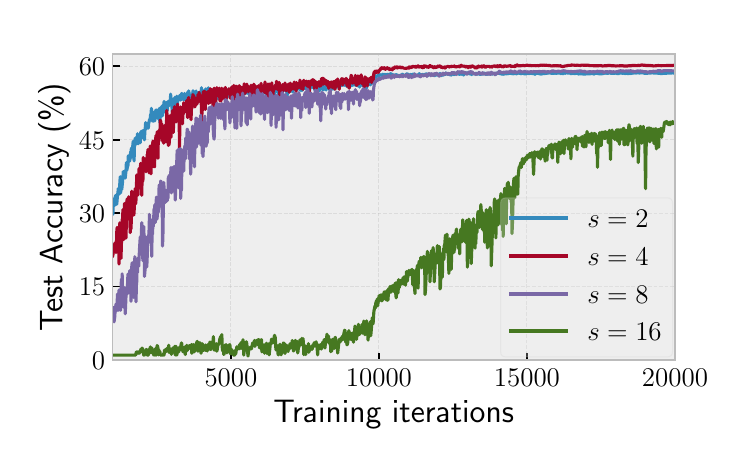}
	\label{fig8ablavgg11acc}}
	\caption{Impact of gradient variability threshold $\delta$.
	Larger $\delta$ biases training toward scaled updates, while a smaller $\delta$ favors unscaled (sensitive) updates.} 
	\label{fig:fig8ablastatic}
\end{figure}

\vspace{0.15cm}
\emph{Impact of gradient variability threshold:} Figure~(\ref{fig:fig8ablastatic}) plots the convergence curves for ResNet and VGG as we vary the threshold metric over multiple scaling factors.
The results show that $\delta$ significantly affects convergence, where a small $\delta$ is conservative and favors unscaled updates while a larger $\delta$ applies scaling more aggressively.
For ResNet50, increasing $\delta$ from 0.5 to 0.8 (Figure~(\ref{fig:res50StepfnAcc}) and (\ref{fig8ablares50acc})) substantially degrades test accuracy.
On the other hand, VGG11 benefits significantly with a smaller threshold of 0.2, where scaling factors 2, 4 and 8$\times$ achieved nearly the same convergence, while $s$ = 16 attained nearly 12\% lesser accuracy (albeit still better than ResNet50 at various $s$).
Empirically, $\delta$ = 0.5 gives a robust balance across all models.

\begin{algorithm}
\DontPrintSemicolon
\SetAlgoLined
\SetKwProg{Fn}{function}{}{}
\caption{\acrfull{ags}}
\label{algo1adapgrads}

\KwIn{Batches ($b_{small}$, $b_{large}$), threshold $\delta$, \gls{lr} $\eta$, \\ 
stability term $\epsilon$ = $10^{-8}$, $\textsl{gradscale} \leftarrow \mathbf{1}$ (per-parameter)}

\For{epoch $e = 1,2,\dots$}{
    \For{iteration $i = 1,2,\dots$}{
        $G^{(i)}_{b_{large}} \leftarrow \nabla f(w^{(i)}, b^{(i)}_{large})$\;
        Compute $\Delta(G^{(i)})$ using Eq.~(\ref{eqn:relgradchange})\;

        \eIf{$\Delta(G^{(i)}) < \delta$}{
            $\widetilde{G}^{(i)} \leftarrow \textsl{gradscale} \odot G^{(i)}_{b_{large}}$
        }{
            $\widetilde{G}^{(i)} \leftarrow G^{(i)}_{b_{large}}$
        }

        $w^{(i+1)} \leftarrow w^{(i)} - \eta \widetilde{G}^{(i)}$\;
    }
    $\textsl{gradscale} \leftarrow \textsc{UpdateGradScale}()$\;
}
\vspace{0.1cm}
\Fn{\textsc{UpdateGradScale}$()$}{
    $s_{max} \leftarrow \sqrt{b_{large} / b_{small}}$\;
    Sample $b^{(i)}_{small} \subset b^{(i)}_{large}$, $|b_{small}| < |b_{large}|$\;
    $G^{(i)}_{b_{small}} \leftarrow \nabla f(w^{(i)}, b^{(i)}_{small})$\;

    \ForEach{parameter $(g_{small}, g_{large})$}{
       $s \leftarrow \min\!\left(\left|\frac{g_{small}}{g_{large} + \;\epsilon}\right|, s_{max}\right)$\;
       \vspace{0.05cm}
        $\textsl{gradscale}.key(p) \leftarrow s$\;
    }
    \KwRet{$\textsl{gradscale}$}
}
\end{algorithm}

\subsubsection{\gls{ags}:}
While static gradient scaling improves accuracy in many cases, it remains sensitive to the choice of scaling factor and can degrade convergence when scaling is overly aggressive.
To address this limitation, we propose a profiling-based \acrfull{ags}, which dynamically transforms large-batch gradients to approximate small-batch behavior.
Unlike static scaling, \gls{ags} applies per-parameter scaling and adapts over the course of training.
As shown in Algorithm~(\ref{algo1adapgrads}), updates from large-batches are computed on every step, while training sensitivity is monitored using the gradient variability metric $\Delta(G^{(i)})$.
When training enters a sensitive phase ($\Delta(G^{(i)}) \ge \delta$), \gls{ags} preserves the original gradients.
Otherwise, updates are scaled using a learned per-parameter factor.
The scaling factors are updated periodically via lightweight profiling, illustrated in the function \textsc{UpdateGradScale}.
A small sub-batch is sampled from the large-batch, and its gradients are compared against large-batch gradients (while adding numerical stability term $\epsilon$) to estimate the relative magnitude per parameter.
This overhead is tolerable if it is executed infrequently; this cost can rise considerably when small and large-batch updates are compared more frequently.

To prevent instability due to stochastic noise, \gls{ags} bounds each scaling factor by $s_{max}$ (line 12) following prior heuristics~\cite{b48,b49}.
This constraint mitigates gradient explosion when the ratio of small to large-batch gradients becomes large.
Conversely, \gls{ags} can also reduce update magnitudes ($s<1$) when large-batch gradients exceeds small-batches— a behavior observed in the later training stages (seen in Figure~(\ref{fig3gradnormres50})).

\emph{Convergence properties}: Under an $L$-smooth loss function and bounded per-parameter scaling factors $s \in [s_{min}, s_{max}]$, \gls{ags} converges at a rate of $\mathcal{O}(1/I)$ for a fixed learning-rate $\eta$ \scalebox{0.85}{\textgreater} $2s_{min}/(Ls_{max}^{2})$, where $I$ is the total training iterations.
When the learning rate decays as $\eta \propto 1/\sqrt{I}$, \gls{ags} achieves the same $\mathcal{O}(1/\sqrt{I})$ convergence-rate as standard \gls{sgd}:

\begin{lemma}[Convergence in Adaptive Gradient Scaling]
\label{lem:ags_convergence}
Assume that the objective function $f$ is $L$-smooth and that stochastic gradients computed using large batches are unbiased with bounded variance, i.e.,
\begin{equation}
\mathbb{E}[g^{(i)} \mid w^{(i)}] = \nabla f(w^{(i)}),
\qquad
\mathbb{E}\|g^{(i)} - \nabla f(w^{(i)})\|^{2} \le \sigma^{2}.
\end{equation}
Let the learning-rate satisfy $\eta$ \scalebox{0.85}{\textgreater} $0$, and let \acrshort{ags} enforce bounded per-parameter scaling factors such that for every iteration $i$ and model parameter $p$,
\begin{equation}
0 < s_{min} < s_{p}^{(i)} \le s_{max}
\end{equation}
If the learning-rate satisfies
\begin{equation}
\eta < \frac{2s_{min}}{L s_{max}^{2}},
\end{equation}
then the iterates generated by \acrshort{ags} satisfy
\begin{equation}
\frac{1}{I} \sum_{i=1}^{I} \mathbb{E}\|\nabla f(w^{(i)})\|^{2}
\le
\frac{f(w^{(1)}) - f^{\ast}}{\epsilon I}
+ \frac{L\eta^{2}s_{max}^{2}\sigma^{2}}{2\epsilon},
\end{equation}
where $f^{\ast} = \inf f$ and
\begin{equation}
\epsilon := \eta s_{min} - \frac{L\eta^{2}s_{max}^{2}}{2} > 0.
\end{equation}
Consequently, \acrshort{ags} converges to a stationary point at rate $\mathcal{O}(1/I)$ for constant $\eta$, and at rate $\mathcal{O}(1/\sqrt{I})$ when $\eta \propto 1/\sqrt{I}$.
\end{lemma}

\begin{proof}
Since $f$ is $L$-smooth, the update rule of \acrshort{ags},
\begin{equation}
w^{(i+1)} = w^{(i)} - \eta \, s^{(i)} \odot g^{(i)},
\end{equation}
satisfies
\begin{equation}
f(w^{(i+1)}) \le f(w^{(i)})
- \eta \langle \nabla f(w^{(i)}), s^{(i)} \odot g^{(i)} \rangle
+ \frac{L\eta^{2}}{2}\| s^{(i)} \odot g^{(i)} \|^{2}.
\end{equation}

Taking conditional expectation with respect to $w^{(i)}$ and noting that $s^{(i)}$ is deterministic relative to the batch randomness yields
\begin{equation}
\begin{aligned}
\mathbb{E}[f(w^{(i+1)}) \mid w^{(i)}]
\;\le\;& f(w^{(i)}) - \eta \,\langle \nabla f(w^{(i)}),s^{(i)} \odot \nabla f(w^{(i)}) \rangle \\
&+ \frac{L\eta^{2}}{2}
\,\mathbb{E}\!\left\|
s^{(i)} \odot g^{(i)}
\right\|^{2}.
\end{aligned}
\end{equation}

Using the boundedness of the scaling factors, we have
\begin{equation}
\langle \nabla f, s^{(i)} \odot \nabla f \rangle \ge s_{min}\|\nabla f\|^{2},
\;\;
\| s^{(i)} \odot g^{(i)} \|^{2} \le s_{max}^{2}\|g^{(i)}\|^{2}.
\end{equation}

Moreover, by the variance assumption,
\begin{equation}
\begin{aligned}
\mathbb{E}\|g^{(i)}\|^{2}
=\;& \|\nabla f(w^{(i)})\|^{2} + \mathbb{E}\!\left\|
g^{(i)} - \nabla f(w^{(i)})
\right\|^{2} \\
\le\;& \|\nabla f(w^{(i)})\|^{2}
+ \sigma^{2}.
\end{aligned}
\end{equation}

Substituting these bounds yields
\begin{equation}
\begin{aligned}
\mathbb{E}[f(w^{(i+1)}) \mid w^{(i)}]
\le\;& f(w^{(i)}) - \eta s_{min}\|\nabla f(w^{(i)})\|^{2} \\
&+ \frac{L\eta^{2}s_{max}^{2}}{2}
\Big(
\|\nabla f(w^{(i)})\|^{2}
+ \sigma^{2}
\Big).
\end{aligned}
\end{equation}

Rearranging terms gives us
\begin{equation}
\begin{aligned}
\Big(
\eta s_{min}
- \frac{L\eta^{2}s_{max}^{2}}{2}
\Big)
\|\nabla f(w^{(i)})\|^{2}
\le\;&
f(w^{(i)}) \\
&- \mathbb{E}[f(w^{(i+1)}) \mid w^{(i)}] \\
&+ \frac{L\eta^{2}s_{max}^{2}\sigma^{2}}{2}.
\end{aligned}
\end{equation}

Choosing $\eta < 2s_{min}/(Ls_{max}^{2})$ ensures the coefficient on the left-hand side is positive.
Setting
\begin{equation}
\epsilon := \eta s_{min} - \frac{L\eta^{2}s_{max}^{2}}{2},
\end{equation}
and taking full expectation and summing over $I$ iterations yields
\begin{equation}
\epsilon \sum_{i=1}^{I} \mathbb{E}\|\nabla f(w^{(i)})\|^{2}
\le
f(w^{(1)}) - \mathbb{E}[f(w^{(I)})]
+ \frac{IL\eta^{2}s_{max}^{2}\sigma^{2}}{2}.
\end{equation}

Setting $f^{\ast} = \inf f$ and dividing by $I\epsilon$ completes the proof.
\end{proof}

	\section{Evaluation}\label{sec:eval}

\subsection{Implementation}
We implement \gls{tula} over PyTorch v2.6.0, with its scaling indicators realized by extending the PyTorch \acrfull{ddp} module.
The \textsl{TulaOptimizer} extends the base \textsl{torch.optim.Optimizer} class to wrap any standard optimizer, and tracks gradient norms across steps and computes gradient variability metric in an online manner.

During the initial exploratory phase, an external model service profiles training runs and collects scaling indicators to construct a performance model.
By default, \gls{tula} employs a partial-search strategy due to its lower overhead and comparable accuracy to full-search strategy (detailed in \S\ref{subsec:perfmodeloptbsz}).
Once an optimal configuration is chosen, training proceeds using large batches with adaptive gradient scaling (AGS).
Scaling factors are updated after every epoch by default, although this frequency is configurable.
\gls{ags} in \gls{tula} is parameterized by a gradient variability threshold $\delta$ and a small reference batch-size used to approximate small-batch updates.

\begin{figure*}
	\centering 
	\subfloat[ResNet50]{\includegraphics[width=0.25\textwidth]{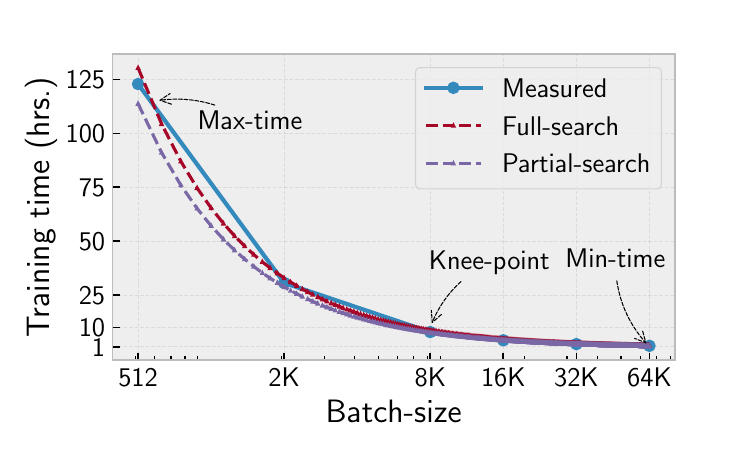}
	\label{fig9eptimeres50}}
	\subfloat[VGG11]{\includegraphics[width=0.25\textwidth]{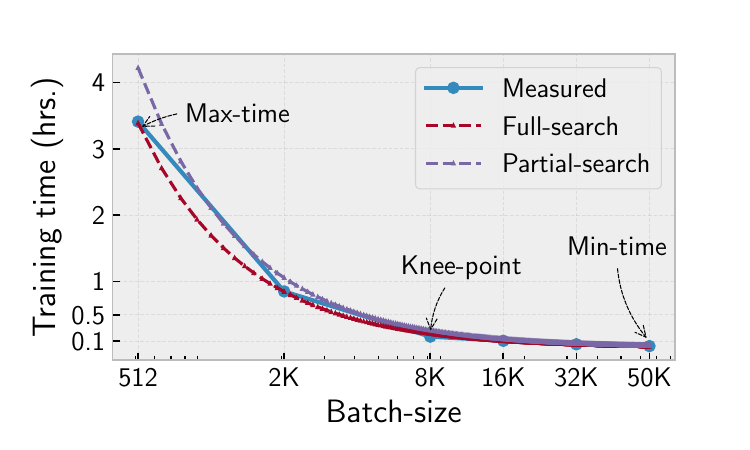}
	\label{fig9eptimevgg11}}
	\subfloat[AlexNet]{\includegraphics[width=0.25\textwidth]{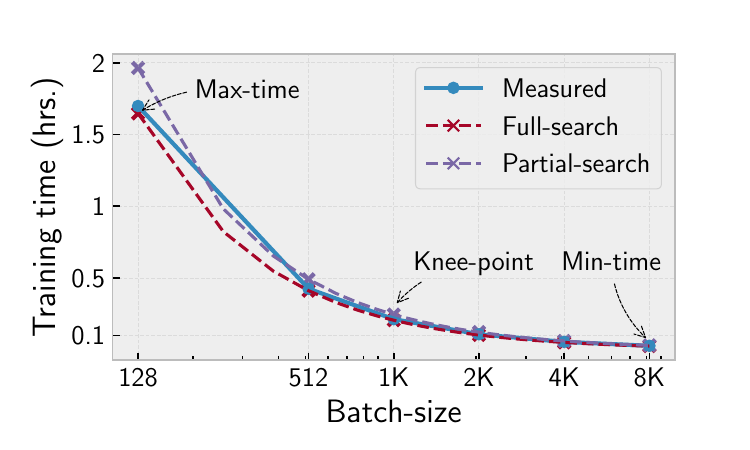}
	\label{fig9eptimealex}}
	\subfloat[MobileNetv3]{\includegraphics[width=0.25\textwidth]{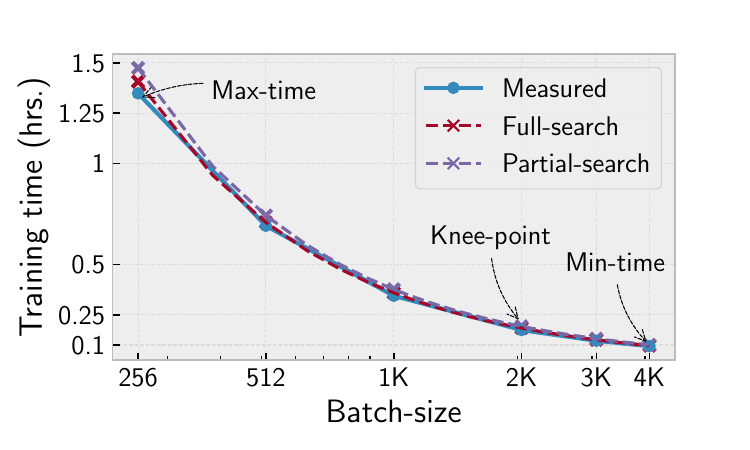}
	\label{fig9eptimemobv3}}
	\caption{Training time vs. batch-size on a 16-node cluster.
	Dashed lines denote predictions from performance models built using full and partial-search.
	Both strategies accurately capture runtime trends and identify near-optimal batch-size.
	%With a static pricing model, minimum cost can be achieved with the longest running time, or vice-versa.
	%The knee-point balances the time-cost trade-offs corresponding to diminishing returns with larger batches.}
	Larger batches offer diminishing returns; training time does not significantly reduce from the knee-point to the minimum-time configuration at the largest batch-size.
	However, the latter suffers considerably more generalization loss.}
	\label{fig9EpochTimeCurves}
\end{figure*}

\subsection{Experimental setup}
We run experiments with varying configurations on up to 16 nodes, with each node having a 48-core Intel Xeon E5-2560 CPU, 128\,GB of memory, and 1 V100S GPU with 32\,GB memory running Ubuntu~22.04.
The nodes are connected via 10Gbps Ethernet and communicate using NCCL v2.21.5 with setting \textsl{NCCL\_ALGO} = \textsl{ring}.
%Although these GPUs are old in the steep development pace of AI/ML, the performance characteristics exhibited here applies to modern devices as well.
%Similarly, bandwidth cost have significantly improved in modern high-speed interconnects such as NVLink, but so has the size of \gls{dnn}s, rendering parallel efficiency and communication as a major distributed training bottleneck.
Although the GPUs used in our experiments are relatively dated given the rapid evolution of AI/ML hardware, with modern accelerators featuring tensor-optimized execution units, high-bandwidth memory, and tightly integrated interconnect fabrics.
While raw compute and memory bandwidth have increased by orders of magnitude, the rapid growth in model scale has preserved communication and memory movement as dominant system bottlenecks in distributed training.
%However, the performance trends observed remain representative of modern accelerators for deep learning workloads.
%Likewise, while communication bandwidth has improved substantially with high-speed interconnects such as NVLink, the scale of deep neural networks (DNNs) has grown proportionally. As a result, communication overhead and parallel efficiency continue to pose significant bottlenecks in distributed training.

We evaluate four models: ResNet50, VGG11, AlexNet and MobileNetv3 over ImageNet-1K, CIFAR100, CalTech101 and CalTech256.
These datasets contain 1.2M, 50K, 9K and 30K images respectively.
ResNet50 on ImageNet reports top-5 while all other \acrshort{dnn}s report top-1 test accuracy.
The models train with cross-entropy loss, \acrshort{sgd} with momentum factor 0.9, initial \gls{lr} of 0.05 in ResNet50 and 0.01 across others.
MobileNetv3 applied weight decay 0.0001, while others used decay factor 0.0005.
The training schedule decays \gls{lr} by 10$\times$ in ResNet50 after [100,\,150,\,200] epochs, AlexNet after [25,\,50,\,75] epochs, and MobileNetv3 after [40,\,80] epochs.
The \gls{lr} in VGG11 is decayed by 5$\times$ after [50,\,75,\,100] epochs respectively.

Across all models, the search space for the performance model sets the minimum and maximum cluster-size to 2 and 16, with candidate configurations increasing exponentially, i.e., [2,\,4,\,8,\,16].
Batch-sizes are bounded with a minimum of 32, while the upper bound is determined by the memory estimation model.
Candidate batches also grow exponentially within this range.
Under full-search, each candidate configuration is briefly profiled, whereas partial-search evaluates only two points (the extremums of batch and cluster-size).

Since memory estimation can be conservative, \gls{tula} may overestimate the largest feasible batch-size and encounter \gls{oom} errors.
\GLS{tula} handles this by frequent job checkpointing and relaunching progressively smaller batches until a feasible configuration is found (akin to PyTorch Lightning's \textsl{auto\_scale\_batch\_size} feature).
Thus, \gls{tula} first relies on its memory model and, if necessary, refines the estimate through runtime probing.
Once the performance model is built and an optimal batch-size chosen based on user objective (minimizing time, cost, or identifying the knee-point), \gls{tula} deploys training with \acrshort{ags} (setting $\delta$ = 0.5 in current experiments) to improve its large-batch generalization quality.

\begin{figure*}
	\centering 
	\subfloat[ResNet50]{\includegraphics[width=0.25\textwidth]{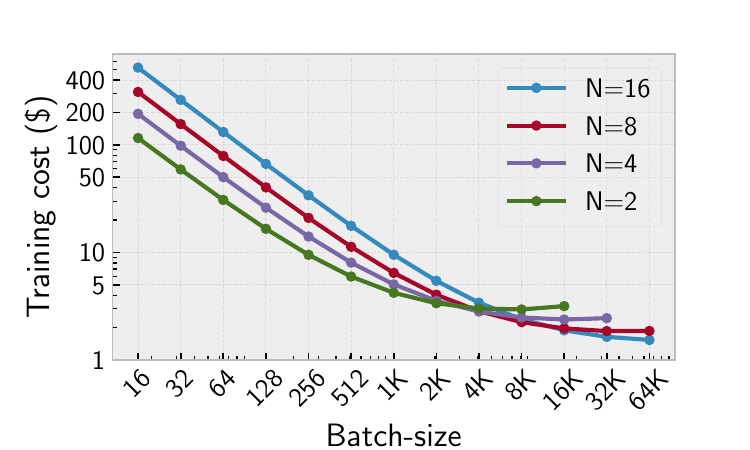}
	\label{fig11memcostres50}}
	\subfloat[VGG11]{\includegraphics[width=0.25\textwidth]{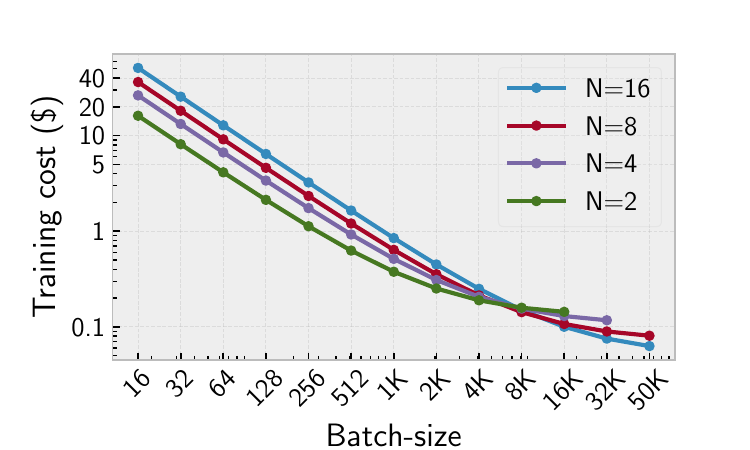}
	\label{fig11memcostvgg11}}
	\subfloat[AlexNet]{\includegraphics[width=0.25\textwidth]{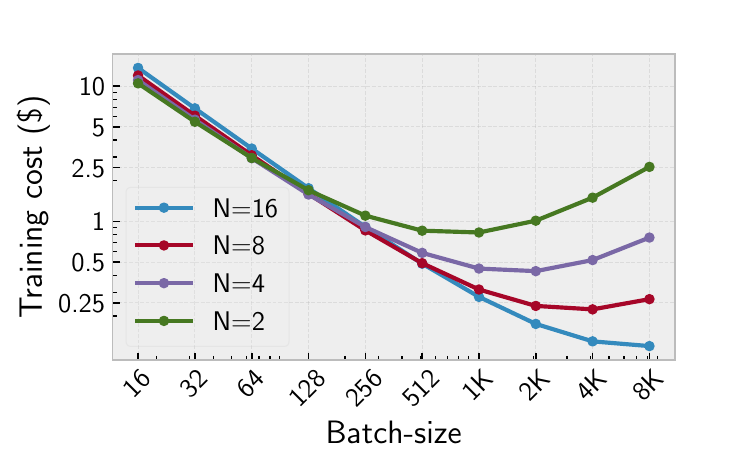}
	\label{fig11memcostalex}}
	\subfloat[MobileNetv3]{\includegraphics[width=0.25\textwidth]{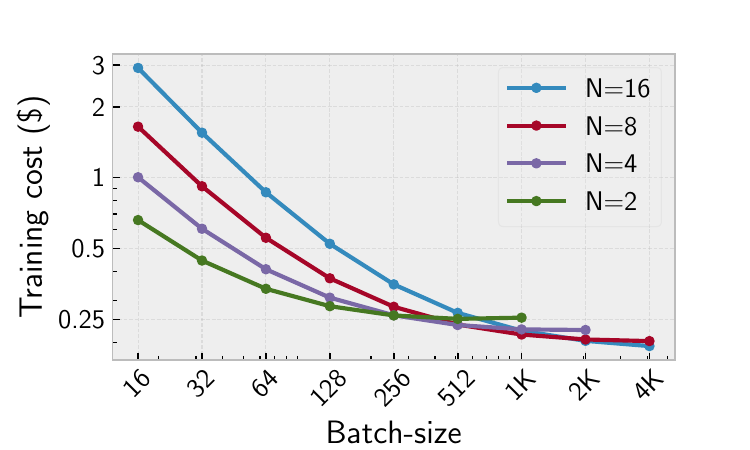}
	\label{fig11memcostmobv3}}
	\caption{Training cost under a memory-based pricing model that charges \$0.15/hour/GB of GPU memory usage.
	The optimal batch-size shifts under this cost model due to increased memory footprint on larger batches.}
	\label{fig11memcosteval}
\end{figure*}

\subsection{Performance Model and Optimal Batch-Size}\label{subsec:perfmodeloptbsz}

\subsubsection{Estimating training time and cost:} Using the performance model described in \S\ref{subsec:designparperf}, we estimate the time and cost of distributed training across a range of global batch-sizes.
Figure~(\ref{fig9EpochTimeCurves}) reports the measured training time along with predictions obtained from performance models constructed via full- and partial-search.
For all \acrshort{dnn}s, the predicted curves closely track the observed runtime trends as batch-size increases.
%Full-search model is developed by briefly deploying all possible configurations in the search space, capturing performance metrics and fitting the model over multiple data-points logged during the profiling phase.
%On the other hand, partial-search model is fitted over the extremums, i.e., the minimum and maximum batches selected apriori.
The full-search model is constructed by exhaustively sampling all configurations within the search space, collecting performance metrics for each, and fitting the model using the multiple data points recorded during the profiling phase.
In contrast, the partial-search model is trained using only the boundary configurations—specifically, the minimum and maximum batch sizes selected a priori—and fitting over these extremal points.
%Due to the additional data available to full-search strategy, its prediction (shown in red) is closer to the ground truth (shown in blue) in Figure~(\ref{fig9EpochTimeCurves}).
Owing to the richer set of profiling data available to the full-search strategy, its predictions (red) more closely aligns with the ground truth (blue) in Figure~(\ref{fig9EpochTimeCurves}).
While full-search yields slightly more accurate estimates, partial-search achieves comparable accuracy with substantially lower profiling overhead, as it evaluates only the boundary configurations.
%As we approach batch ceilings (right side of the curve), the estimated time/cost of both strategies saturates to nearly the same scale.
As the batch-size approaches its upper bound (i.e., the rightmost region of the curve), the estimated time or cost predicted by both strategies converges and saturates at nearly the same level.
%% SAHILLLLLLLLL: add partial/full search discussion here!

Assuming a fixed pricing model (see Equation~(\ref{eqn:trainCostEst})), training cost follows the same trajectory as execution time.
This reflects common cloud pricing practices where compute resources are charged at a constant hourly rate.
For example, a Google Cloud virtual machine with a single NVIDIA V100 GPU costs \$2.48 per hour in the US region in 2025.
Under such a pricing model, \gls{tula} enables users to select the batch-size that minimizes training cost or identify the knee-point beyond which increasing the batch-size yields diminishing returns in terms of time savings.
For the 16-node configuration, this knee-point corresponded to a batch-size of 8192 for ResNet50 and VGG11, 1024 for AlexNet and 2048 for MobileNetv3.
Beyond these points, further increases in batch-size does not significantly reduce training time/cost, but may further exacerbate the generalization gap if the chosen batch-size is too large.

\subsubsection{Memory-based pricing model:} Training cost is ultimately dictated by the underlying pricing model, which may depend on resource utilization rather than the available node count.
To evaluate such scenarios, we consider an alternative pricing strategy in which \gls{gpu}s are charged based on their memory usage.
Under this cost model, a V100 \gls{gpu} instance incurs cost $\$0.15 / \text{hour} / \text{GB}$ of memory.

We construct the performance model using partial-search and estimate the cost of training as a function of batch-size under memory-based pricing for cluster-size configurations [2,\,4,\,8,\,16].
Figure~(\ref{fig11memcosteval}) illustrates the resulting cost surfaces.
In contrast to fixed-rate pricing, the optimal batch-size shifts due to the increased memory footprint associated with larger batches.
At 16 nodes, the cost-optimal batch size is 16384 for ResNet50 and VGG11, 2048 for AlexNet, and 1024 for MobileNetv3.
In some cases like AlexNet (Figure~(\ref{fig11memcostalex})), cost may increase at larger batches when additional activation memory cost outweighs the savings from allocating fewer nodes.
Thus, training on a 16 node cluster becomes cheaper than training on 2, 4 or 8 for batches 1024 or greater.
%outweighs gains in parallel efficiency.
These results highlight the importance of jointly considering system performance, memory usage, and pricing models when selecting the batch-size in distributed training.

\begin{figure}
    \centering
    \hspace*{-0.3cm}% adjust this value to move more/less to the left
    \subfloat[Memory prediction error]{
        \includegraphics[width=0.25\textwidth]{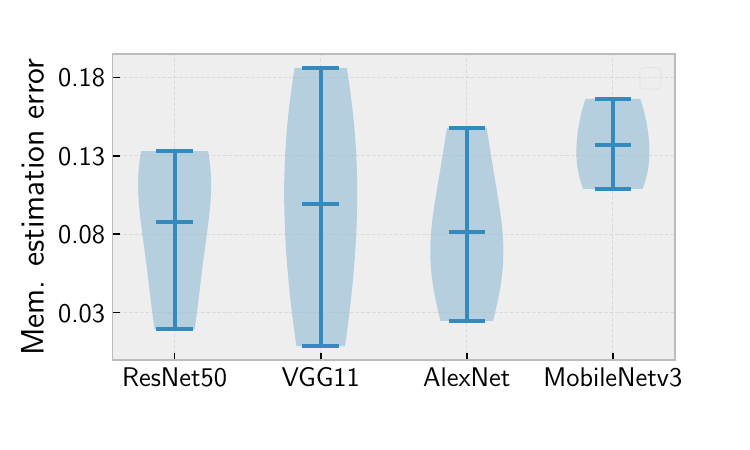}
        \label{fig10memprederr}}
    \subfloat[Time prediction error]{
        \includegraphics[width=0.25\textwidth]{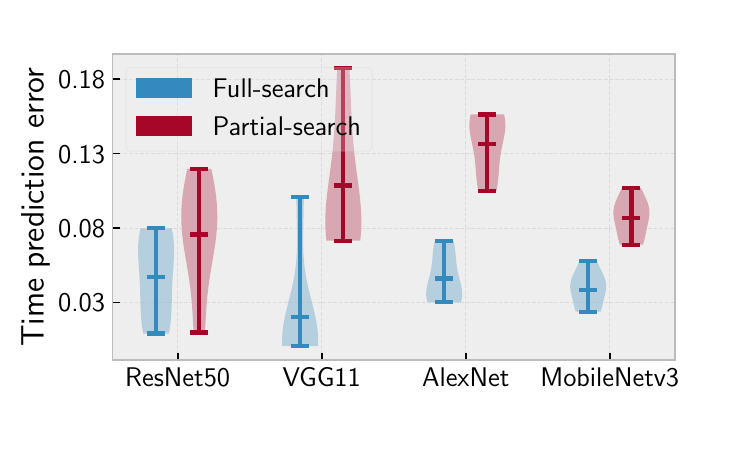}
        \label{fig10timeprederr}}
    \caption{Relative error of \gls{tula} in estimating \gls{gpu} memory and training time.}
    \label{fig10prederr}
\end{figure}

\subsubsection{Estimation accuracy:} The accuracy of \gls{tula}'s performance models is evaluated by comparing predicted memory usage and training time against observed values.
Figure~(\ref{fig10memprederr}) reports the relative error of the memory estimation model across all batch-sizes explored during search phase (shown here for all local batches corresponding to $N$ = 16 from Figures~(\ref{fig9EpochTimeCurves}) and (\ref{fig11memcosteval})).
\emph{Across all the evaluated neural networks, the median memory prediction error ranges only between 8--14\%}.
When estimation error does indeed result in out-of-memory failures, \gls{tula} automatically falls back to the next feasible configuration and redeploys the job, ensuring robustness in practice.

Figure~(\ref{fig10timeprederr}) shows the relative error in training-time prediction for models constructed using full and partial-search strategies.
Full-search achieves higher accuracy (and lower error) due to more extensive profiling across configurations, \emph{yielding median errors below 2--5\%}.
Partial-search, which profiles only the boundary configurations, \emph{incurs higher median error but remains within 7.5--14\% across all models}.
Despite its slightly lower accuracy, partial-search reliably captures relative performance trends across batch and cluster configurations.

It is important to note that \gls{tula} is designed to identify configurations that optimize training time, cost, or knee-point behavior rather than to minimize absolute prediction error.
Our results show that, even with partial-search, \gls{tula} consistently identifies the ideal configuration that optimizes for parallel efficiency under different scaling strategies and objectives.

\subsection{Statistical Performance with Large-Batches}\label{sub:statperfeval}

\begin{table}
\centering
\caption{Top-1 test accuracy with vanilla large-batch training (LBT), LARS, Post-local SGD (PoLo), \gls{lr} scaling (LRS), and \gls{tula}'s \gls{ags}.
For every batch-size, \textcolor{green!60!black}{\textbf{green}} $\rightarrow$ high, 
\textcolor{red!60!black}{\textbf{red}} $\rightarrow$ low, 
\textcolor{yellow!60!black}{\textbf{yellow}} $\rightarrow$ intermediate accuracy.}
\label{table:larbatcheval}
\resizebox{\columnwidth}{!}{
\begin{tabular}{|c|c|c|c|c|c|c|}
    \hline
    \bfseries Model & \bfseries Batch-size 
    & \multicolumn{5}{c|}{\bfseries Test accuracy (\%)} \\
    \cline{3-7}
    & & \bfseries LBT & \bfseries LARS & \bfseries PoLo & \bfseries LRS & \bfseries \emph{\gls{tula}} \\
    \hline

    \multirow{3}{*}{ResNet50} 
    & 1024 
    & \cellcolor{mid}79.67 
    & \cellcolor{mid}83.26 
    & \cellcolor{worst}74.03 
    & \cellcolor{mid}83.33 
    & \cellcolor{best}\emph{85.28} \\
    \cline{2-7}

    & 2048 
    & \cellcolor{mid}75.55 
    & \cellcolor{best}85.02
    & \cellcolor{worst}65.58 
    & \cellcolor{mid}74.44 
    & \cellcolor{mid}\emph{84.58} \\
    \cline{2-7}

    & 4096 
    & \cellcolor{mid}64.72 
    & \cellcolor{best}86.77 
    & \cellcolor{worst}62.46 
    & \cellcolor{mid}70.11 
    & \cellcolor{mid}\emph{82.16} \\
    \hline

    \multirow{3}{*}{VGG11} 
    & 1024 
    & \cellcolor{mid}83.30 
    & \cellcolor{mid}52.93 
    & \cellcolor{mid}82.90 
    & \cellcolor{worst}7.50 
    & \cellcolor{best}\emph{84.42} \\
    \cline{2-7}

    & 2048 
    & \cellcolor{mid}75.92 
    & \cellcolor{mid}50.91 
    & \cellcolor{mid}81.13 
    & \cellcolor{worst}5.00 
    & \cellcolor{best}\emph{82.58} \\
    \cline{2-7}

    & 4096 
    & \cellcolor{mid}70.24 
    & \cellcolor{mid}48.64 
    & \cellcolor{mid}79.50 
    & \cellcolor{worst}5.00 
    & \cellcolor{best}\emph{81.14} \\
    \hline

    \multirow{3}{*}{AlexNet} 
    & 256 
    & \cellcolor{mid}80.62 
    & \cellcolor{worst}55.40 
    & \cellcolor{best}85.11
    & \cellcolor{mid}83.38 
    & \cellcolor{mid}\emph{84.44} \\
    \cline{2-7}

    & 512 
    & \cellcolor{mid}72.80 
    & \cellcolor{worst}56.00 
    & \cellcolor{best}85.95
    & \cellcolor{mid}72.30 
    & \cellcolor{mid}\emph{83.00} \\
    \cline{2-7}

    & 1024 
    & \cellcolor{mid}62.60 
    & \cellcolor{mid}57.00 
    & \cellcolor{best}85.67
    & \cellcolor{worst}32.74 
    & \cellcolor{mid}\emph{81.20} \\
    \hline

    \multirow{3}{*}{MobileNetv3} 
    & 1024 
    & \cellcolor{mid}76.80 
    & \cellcolor{worst}67.64 
    & \cellcolor{mid}81.14 
    & \cellcolor{mid}78.66 
    & \cellcolor{best}\emph{81.56} \\
    \cline{2-7}

    & 2048 
    & \cellcolor{mid}71.15 
    & \cellcolor{mid}72.82 
    & \cellcolor{mid}75.54 
    & \cellcolor{worst}10.00 
    & \cellcolor{best}\emph{79.68} \\
    \cline{2-7}

    & 4096 
    & \cellcolor{mid}66.48 
    & \cellcolor{mid}74.29 
    & \cellcolor{mid}73.91 
    & \cellcolor{worst}5.00 
    & \cellcolor{best}\emph{75.70} \\
    \hline
\end{tabular}
}
\end{table}

\subsubsection{Evaluation methodology:}
From the performance models in Figure~(\ref{fig9EpochTimeCurves}), \gls{tula} identifies knee-points on the time--batch-size trade-off curves as the optimal configuration.
%This corresponds to 4K for ResNet50 and VGG11, 1K for AlexNet and 2K for MobileNetv3.
\emph{In this section, we evaluate convergence quality of models over batches around these knee-points.
Beyond these points, further batch scaling provides diminishing returns in time or cost reduction}.
%Therefore, we evaluate the statistical performance around these batches}.

We compare \gls{tula}'s \gls{ags} against popular techniques: vanilla large-batch training (LBT), learning-rate scaling (LRS)~\cite{b11}, Post-local SGD (PoLo)~\cite{b51}, and LARS~\cite{b14}.
LRS and LARS apply linear learning-rate scaling relative to a base batch-size of 128, while PoLo switches from synchronous to local-\gls{sgd} after 5 epochs, aggregating updates 4 times every epoch.
All these methods apply \gls{lr} warmup for 5 epochs.

\subsubsection{Results:}
Table~(\ref{table:larbatcheval}) shows that \gls{tula} consistently mitigates the generalization degradation observed with vanilla large-batch training (LBT) across all models and batch-sizes.
In LBT, test accuracy deteriorates in correspondence to larger global batches on account of the generalization gap in large-batch training.
In VGG11 and MobileNet, \acrshort{ags} saw the best accuracy across all evaluated batches.
For ResNet50 and AlexNet, \gls{ags} remains competitive and avoids severe accuracy degradation at larger batches, although LARS and PoLo observed slightly higher accuracy.
%In contrast, LRS exhibits unstable behavior at large batch-sizes, with accuracy degrading significantly as the scaled \gls{lr} becomes excessively large.
%The scaling factor in LRS is directly proportional to the ratio of large and small batch-size of reference, thus scaling to a very large batch with a small base batch-size can result in a very high scaling factor.
%Additionally, unlike \gls{ags}, LRS applies this scaling factor indiscriminately on every iteration, thus destabilizing the training phase.
In contrast, LRS demonstrates unstable behavior at large batches, with accuracy deteriorating markedly as the scaled \gls{lr} grows excessively large. The scaling factor in LRS is directly proportional to the ratio between the target large batch and the reference small batch-size; consequently, scaling to very large batches from a small base batch-size can produce an excessively high multiplier.
Moreover, unlike \gls{ags}, LRS applies this scaling factor uniformly at every iteration, regardless of the optimization landscape, which can further destabilize the training process (as specially seen in models like VGG11 and MobileNetv3).
Although LARS performed well on ResNet50, it did not converge well on other models.
This is because LARS (which scales the \gls{lr} by $||w||/||\nabla w||$) was originally developed for residual networks~\cite{b14}, while models like VGG11, AlexNet and MobileNetv3 have different characteristics.
Without residual paths, large per-layer scaling can significantly cause gradient instability, while depthwise layers (such as those in MobileNetv3) have very small parameter norms that can produce extreme scaling ratios in LARS, resulting in unstable \gls{lr} that destabilizes training.
In case of PoLo, synchronization frequency plays a crucial role in model convergence.
Reduced synchronization can lead to model drift, and converging local models to sharp minima even before model aggregation.
We thus see low to intermediate accuracy in most cases, with the exception of AlexNet.
We suspect that PoLo performed well here due to the model's small size with fewer deep non-linearity over a modest dataset like CalTech101.
In contrast, ResNet50 training over ImageNet saw least accuracy with PoLo as reduced synchronization deprived the model of beneficial stochasticity.

Overall, while existing techniques perform well only for specific models or batch regimes, \gls{tula}'s \acrshort{ags} provides robust and consistently high accuracy, making it well-suited for system-driven batch-size optimization in large-scale distributed training.

	\section{Related Work}\label{sec:relatedwork}

\gls{tula} lies at the intersection of distributed training systems and large-batch optimization.
Its design addresses two complementary challenges: (i) selecting the efficient batch-size and cluster configurations for distributed training on shared or cloud infrastructure~\cite{b52}, and (ii) mitigating the statistical degradation that often accompanies large-batch training.
The optimal batch-size depends jointly on the model, dataset, hardware configuration, and user objective (e.g., minimizing time, cost, or identifying a knee-point), thus motivating system-level approaches that reason across these dimensions.

Several prior systems jointly consider parallel efficiency and statistical behavior when training deep models at scale.
Pollux~\cite{b6} dynamically schedules training jobs on heterogeneous clusters based on performance modeling, while KungFu~\cite{b9} adapts resource allocation in response to observed training dynamics.
AdaScale SGD~\cite{b53} adjusts the effective \gls{lr} based on gradient variance to improve scaling efficiency.
In contrast to these approaches, \gls{tula} explicitly models batch-size effects on both performance and generalization, enabling principled batch selection prior to full-scale training.

Memory-aware training has also received significant attention.
ZeRO~\cite{b46} analyzes optimizer and model states to reduce memory footprint and enable larger batches, particularly for transformer models.
PyTorch Lightning’s \textsl{auto\_scale\_batch\_size} feature uses iterative probing to identify the largest batch-size that fits in \gls{gpu} memory.
\gls{tula}'s memory estimation and checkpoint-and-redeploy mechanism is complementary to these efforts, allowing it to recover gracefully from estimation errors while minimizing profiling overhead.
Prior work has also explored training sensitivity to reduce communication costs in distributed training~\cite{b41,b58,b39}.

A large body of work focuses on mitigating the generalization gap observed in large-batch training.
Common techniques include learning-rate scaling and warmup~\cite{b11}, batch and layer normalization, adaptive learning-rates, and gradient clipping.
\gls{lr} scaling (LRS) increases the learning-rate proportional to the batch-size, while other related methods co-tune batch-size and \gls{lr}~\cite{b48,b55}.
Other approaches introduce noise or extrapolation to avoid convergence to sharp minima~\cite{b47,b56}.
Post-local SGD (PoLo)~\cite{b51} begins with synchronous training and later switches to local updates to improve generalization at scale.
Layer-wise optimizers such as LARS~\cite{b14} and LAMB~\cite{b57} further adapt \gls{lr} based on layer-wise gradient and weight magnitudes.

These optimizations are orthogonal and complementary to \gls{tula}'s adaptive gradient scaling.
\gls{tula} provides a system-level framework for selecting optimal configurations and stabilizing training at scale, into which alternative large-batch optimization strategies can also be readily integrated.

	\section{Current Limitations and Discussion}\label{sec:limitations}

While \gls{tula} demonstrates strong improvements in both parallel and statistical efficiency, it has certain limitations that motivate future extensions.
\gls{tula}'s performance model assumes relatively stable communication characteristics.
In shared or congested networks, communication costs can vary significantly over time, potentially reducing prediction accuracy.
We plan to explore techniques to adapt the performance model to dynamic network conditions in the future.

While partial-search significantly reduces profiling cost over full-search, this overhead may still be non-trivial in environments with long job queue times or strict scheduling constraints, like shared \gls{hpc} systems.
Our evaluation considers only fixed-rate and memory-based pricing schemes.
In practice, cloud pricing can involve additional factors such as spot instances and heterogeneous hardware.
\gls{tula} does not currently incorporate such dynamic or multi-dimensional cost models.

The convergence behavior of \gls{tula} depends on hyperparameters of \gls{ags}, including the gradient variability threshold~$\delta$, the choice of reference small batch-size, and the frequency of per-parameter scaling update that compares small and large-batch gradients.
Setting extreme values on \gls{ags}, such as a very small reference batch-size and a very high large-batch could destabilize updates if the scaling factor is high (by making $s_{max}$ and $g_{small}$ in Algorithm~(\ref{algo1adapgrads}) very high).
The threshold $\delta$ in \gls{ags} governs whether updates are rescaled or applied without modification; setting it too high may unnecessarily scale large-batch updates, even when the underlying optimization landscape does not exhibit steep curvature or sharp gradients.
On the other hand, setting a very small $\delta$ may not scale updates at all, effectively acting as vanilla large-batch training.
Through empirical evaluation, we found $\delta$ = 0.5 to provide the best trade-off, and therefore adopt it as the default setting in our experiments.

While we observed that fixed parameter values are effective across all the evaluated models and datasets, these settings may require adjustment for different architectures, datasets, or optimization regimes.
In future, we plan to incorporate a checkpoint-based auto-tuning mechanism that monitors validation accuracy during training and prunes configurations that negatively impact convergence quality.

Our current evaluation focuses on convolutional models for vision workloads. Although \gls{tula} is architecturally model-agnostic, extending it to other domains like large language models (LLMs) and graph neural networks (GNNs) introduces additional challenges, including different scaling and generalization characteristics.
In particular, memory estimation for LLMs needs to account for factors beyond batch-size, such as sequence length, transformer depth, and hidden dimensions.
Additionally, layer heterogeneity in LLMs can exhibit varying gradient statistics.
For e.g., attention layers can have sharp curvature, layernorm parameters have smaller but sensitive gradients, while embedding layer can have highly sparse gradients~\cite{b71}.
Thus, it may not be intuitive to apply gradient scaling uniformly across all model layers in this case, as it may disturb layernorm statistics or over-amplify noisy embedding gradients.
Further, AdamW optimizer (popularly used in LLM training) combined with gradient scaling, would effectively modify the diagonal preconditioner and may lead to unexpected behavior.

On the other hand, gradient noise in GNN like architectures is graph structure dependent, so performing gradient scaling in a static way may amplify structural bias and degrade convergence quality.
%Furthermore, we plan to explore more expressive gradient-scaling strategies by leveraging auxiliary deep models (e.g., transformers or autoencoders) to guide \gls{ags}, similar in spirit to recent work on learning-based optimization~\cite{b69,b70}.
To address these challenges, we are exploring more expressive gradient-scaling strategies by leveraging auxiliary deep models (e.g., transformers or autoencoders), effectively acting as a knowledge distillation process (mapping large-to-small batch gradient space), similar to recent works on learning-based optimization~\cite{b69,b70}.

	\section{Conclusion}\label{sec:conclusion}

Training neural networks efficiently on distributed resources requires carefully balancing parallel efficiency, resource cost, and statistical performance.
These factors interact in complex and often non-obvious ways, making na\"ively choosing batch-size and cluster-size configuration sub-optimal.
\gls{tula} addresses this challenge by providing an online, profiling-driven service that models training performance across different configurations, identifies optimal operating points, and improves convergence quality in large-batch regimes.

Our evaluation shows that \gls{tula}'s parallel performance model accurately captures training time and cost trends, enabling it to identify the batch-size that best satisfies user-specified objective such as minimal training time or cost.
A more balanced and practical approach is to identify the batch-size corresponding to the knee-point of the pareto frontier.
Beyond this point, increasing the batch-size yields diminishing returns in efficiency while exacerbating the generalization gap.
By selecting a batch-size at or near this knee-point, \gls{tula} avoids unnecessary resource usage and degraded model quality.
Since the chosen batch-size can be prone to generalization gap relative to smaller training batches, the proposed \gls{ags} technique improves final model quality over vanilla large-batch training.

Across multiple vision workloads and convolutional models, batch configurations selected by \gls{tula} achieves substantial performance gains over na\"ively chosen settings, with up to 16$\times$ speedup observed in ResNet50, 20$\times$ in VGG11, 8$\times$ in AlexNet and 7.7$\times$ in MobileNetv3.
In addition, \gls{tula}'s adaptive gradient scaling (\acrshort{ags}) consistently mitigates the generalization gap associated with large-batch training, improving final test accuracy by an average of 8.8\% compared to vanilla large-batch baselines for the same effective batch-size.

	\section*{Acknowledgments}

This research used resources of the Oak Ridge Leadership Computing Facility at the Oak Ridge National Laboratory, which is supported by the Office of Science of the U.S. Department of Energy under Contract No. DE-AC05-00OR22725.

\end{document}